\documentclass{article}

\usepackage[final]{corl_2026} 

\usepackage{silence}

\usepackage[T1]{fontenc}

\usepackage{graphicx}
\usepackage{booktabs}
\usepackage{float}
\usepackage{adjustbox}
\usepackage{multirow}
\usepackage{makecell}
\usepackage{tabularx}
\usepackage{enumitem}
\usepackage{wrapfig}
\usepackage{xspace}
\usepackage{subcaption}
\usepackage{amsmath, amssymb, amsthm}
\newtheorem{theorem}{Theorem}

\newtheorem{remark}{Remark}
\newtheorem{lemma}{Lemma}

\newtheorem{assumption}{Assumption}
\newtheorem{definition}{Definition}

\graphicspath{{./figure/}}

\usepackage{xcolor}
\definecolor{citecolor}{HTML}{0071bc}
\definecolor{googleblue}{HTML}{4285F4}
\definecolor{googlered}{HTML}{EA4335}
\usepackage[table]{xcolor}
\definecolor{lightgreen}{RGB}{240, 251, 237}
\definecolor{lightblue}{RGB}{237, 245, 251}
\definecolor{lightred}{RGB}{251, 240, 237}
\definecolor{lightgrey}{RGB}{240,240,240}

\hypersetup{
    pdftitle={Where Success Breaks: Failure-Boundary Learning for Robust Vision-Language-Action Models},
    pdfauthor={Yanzhe Chen, Zhijun Cao, Mike Zheng Shou},
    colorlinks=true,
    breaklinks=true,
    citecolor=citecolor,
    linkcolor=googlered,
    urlcolor=googleblue,
}

\usepackage{cleveref}

\newcommand{\eg}{\textit{e.g.,}\xspace}
\newcommand{\ie}{\textit{i.e.,}\xspace}

\title{Where Success Breaks: Failure-Boundary Learning for Robust Vision-Language-Action Models}

\vspace{-20pt}

\author{
  Yanzhe Chen$^*$, Zhijun Cao$^*$, Mike Zheng Shou$^\dagger$ \\[0.1em]
  Showlab, National University of Singapore \\[0.1em]
  $^*$ Equal contribution. \quad $^\dagger$ Corresponding author.
}

\begin{document}
\maketitle

\begin{figure}[htb]
  \vspace{-20pt}
  \centering
    \includegraphics[width=\linewidth, trim=0 0 0 0, clip]{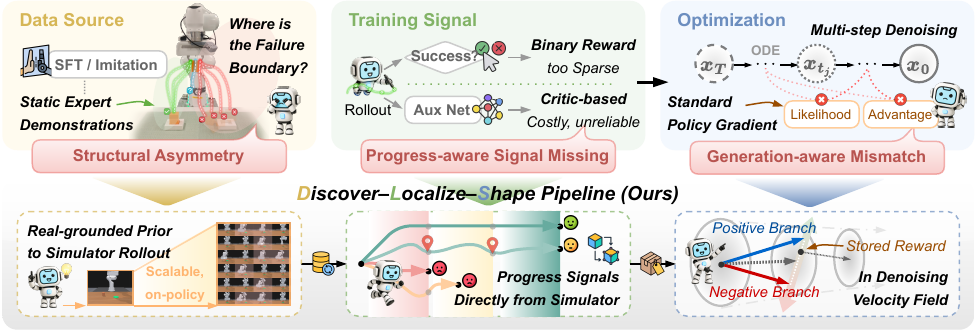}
    \caption{\textbf{Failure-Boundary Learning for Flow-based VLAs.}
    Three bottlenecks cripple current post-SFT adaptation: \textit{asymmetric supervision}, \textit{missing progress signal}, and \textit{generation-aware mismatch} (top). Our DLS pipeline resolves them in turn (bottom): scalable on-policy rollouts of a real-grounded prior, simulator-derived progress labels, and push--pull shaping in the velocity field.}
  \label{fig:intro}
  \vspace{-8pt}
\end{figure}

\begin{abstract}
Vision-language-action (VLA) models adapted through supervised fine-tuning (SFT) inherit a structural asymmetry: expert demonstrations teach the policy where success behavior lies, but provide no signal about where it ceases to be reliable.
We argue that robust VLA adaptation should therefore be viewed not as further demonstration fitting, but as \textbf{Failure-Boundary Learning}---the problem of \textit{\textbf{D}iscovering}, \textit{\textbf{L}ocalizing}, and \textit{\textbf{S}haping} the boundary between recoverable deviations and task failure. To instantiate this view, we propose \textbf{\textit{DLS}}: built on a \textbf{\textit{real-grounded behavioral prior}} from few real demonstrations and simulated co-training, DLS \textit{discovers} failure boundaries at scale through on-policy digital twin rollouts. Rather than reducing each rollout to a binary label, \textbf{\textit{semantic progress localization}} uses privileged simulator states to assign progress-aware signals that capture \textit{where} the failure boundary is crossed, not merely \textit{whether}. These signals drive \textbf{\textit{directional boundary shaping}} in the flow dynamics---reinforcing success-producing denoising directions and suppressing failure-producing ones, without action likelihoods or auxiliary critics. Across real-robot manipulation tasks, DLS improves robustness over SFT and online RL baselines, especially under randomized initial states and unseen visual conditions.

\end{abstract}


\section{Introduction}
\label{sec:introduction}

Vision-language-action (VLA) models~\cite{brohan2022rt1,zitkovich2023rt2,kim2024openvla,black2024pi_0,bjorck2025gr00t,intelligence2025pi_05,team2025gemini,mees2024octo} have become a strong foundation for robot manipulation, yet adapting them to downstream tasks still largely relies on supervised fine-tuning (SFT) over expert demonstrations. While effective, it inherits a structural asymmetry of imitation learning~\cite{ross2011dagger,kachaev2025dontblind,xiang2025vlaposttraining,zhang2025robustvla}: it pulls the policy toward successful actions, while providing little signal about actions that lead to breakdown. During closed-loop execution, small errors can move the robot outside the demonstration distribution, where imitation offers limited guidance on how to recover or what to avoid~\cite{xia2025phoenix,lin2025failsafe,welte2026flowcorrect,kalinowska2021ergodic}.
The core limitation is therefore not insufficient exposure to successful behavior, but the absence of any signal about where that behavior ceases to be reliable.

We study this missing signal in the post-SFT adaptation of flow-based VLAs~\cite{black2024pi_0,intelligence2025pi_05,lipman2022flow}, and frame it as \textbf{Failure-Boundary Learning}: learning where the policy's own closed-loop behavior transitions from recoverable deviation to task failure, a boundary we locate from privileged simulator state to label each rollout. This view is more specific than simply learning from failures. A useful failure signal should reveal not only \emph{whether} an episode failed, but also \emph{where} it failed, how far the policy progressed before breakdown, and which partial behaviors remain worth preserving. Robust adaptation therefore requires three operations: \textbf{\textit{(\romannumeral1) discover}} self-generated failures under closed-loop execution, \textbf{\textit{(\romannumeral2) localize}} where the rollout breaks down in task progress, and \textbf{\textit{(\romannumeral3) shape}} the policy away from failure-producing directions.

This perspective imposes three requirements on the training signal. It should be \textbf{on-policy and scalable}---manually curated failures rarely cover errors the current policy produces, and boundary discovery requires many trials beyond what physical robot interaction can provide~\cite{luo2024serl,intelligence2025pi06*}; \textbf{progress-localized}, since failures at different task stages imply different corrective signals; and \textbf{generation-aware}, since flow-based VLAs generate actions via multi-step denoising, where exact action likelihoods are not directly available for standard policy-gradient updates~\cite{liu2025flowgrpo,chen2025pirl}. Together, these demands call for a procedure that exposes failures at scale, locates where each rollout breaks down, and turns this signal into policy improvement without likelihoods or value functions~\cite{zheng2025diffusionnft,wang2026pistepnft}.

We instantiate failure-boundary learning with a compact \textbf{Discover–Localize–Shape (DLS)} pipeline (Fig.~\ref{fig:intro}), grounded in a \textbf{\textit{real-grounded behavioral prior}} obtained through Sim-Real co-training with few real demonstrations and simulation for behavioral diversity. Building on this prior, we use the digital twin as a \textbf{boundary discovery} instrument: rolling out the current policy at scale exposes failures induced by its own closed-loop behavior. To trace these failures, \textbf{\textit{semantic progress localization}}---where ``semantic'' names the subgoal identity (\eg grasped, placed), not a source of understanding---uses privileged simulator states to assign progress-aware signals that distinguish clean successes, inefficient successes, and partial failures. These signals drive \textbf{\textit{directional boundary shaping}} in the flow dynamics, reinforcing success-producing denoising directions while suppressing failure-producing ones. Across real-robot manipulation tasks, our method improves robustness over SFT and online RL baselines under randomized initial states and unseen visual conditions.

Our contributions are summarized as follows:
\vspace{-0.15cm}
\begin{itemize}[leftmargin=*]
    \item \textbf{Boundary-centric adaptation.}
    We cast post-SFT adaptation of flow-based VLAs as learning where closed-loop behavior breaks, rather than only fitting where expert behavior succeeds.
    \item \textbf{Failures with progress.}
    We convert self-generated rollouts into progress-localized supervision, separating clean successes, inefficient successes, and partial failures.
    \item \textbf{Critic-free flow shaping.} We shape flow-based VLAs with real-grounded, simulator-scaled push--pull updates, without action likelihoods or learned critics.
\end{itemize}

\section{Related Work}
\label{sec:relatedwork}

\subsection{Post-SFT Optimization for Flow-Based VLAs}
\label{sub:vla_models}

Adapting pretrained VLA policies~\cite{brohan2022rt1,zitkovich2023rt2,kim2024openvla,black2024pi_0,bjorck2025gr00t,intelligence2025pi_05,mees2024octo,zhang2025dreamvla,li2024cogact,cen2025worldvla} beyond SFT~\cite{ross2011dagger,chen2026aca,kim2025openvlaoft} requires moving past the covariate-shift and outcome-blind nature of imitation~\cite{pan2025muchado,romer2026failure,xu2025can}, and existing post-SFT methods for flow-based VLAs differ along two axes: how to optimize without tractable action log-likelihoods, and what supervisory signal to optimize against. \textbf{\textit{(\romannumeral1) Likelihood-Recovery RL}} discretizes the denoising SDE to enable PPO- or GRPO-style updates, at the cost of solver coupling and memory overhead~\cite{liu2025flowgrpo,chen2025pirl,zhang2025reinflow,li2025simplevlarl,liu2026rl4vla}. \textbf{\textit{(\romannumeral2) Critic-Driven Shaping}} learns value or reward critics that bypass likelihoods but suffer from critic overfitting and reward hacking~\cite{li2025grrl,prabhudesai2023alignprop,shu2025rftf,skalse2022defining}. \textbf{\textit{(\romannumeral3) Preference-Based Ranking}} avoids both via contrastive comparison, with step-wise extensions reaching into the denoising trajectory~\cite{zheng2025diffusionnft,wang2026pistepnft,rafailov2023dpo,zhang2025grape}, yet ranks rollouts by trajectory-level outcomes rather than by where they fail. \textbf{\textit{(\romannumeral4) Stage-Aware Reward Design}} decomposes manipulation into ordered sub-stages via  planners~\cite{dalal2024plan,chen2025robohorizon,zhang2025reinbot} or rule-based event predicates~\cite{chen2025sarm,xu2025stare,escoriza2025multi}, but these signals are mostly materialized as per-step potential-based shaping~\cite{kim2025stage} consumed by PPO with value networks, or as trajectory-level preferences~\cite{xu2025stare}. Across these families, dense supervision rarely reaches online step-wise ranking, while preference ranking retains simplicity at the price of outcome-level labels; we close this gap by injecting trajectory-level progress-aware labels into step-wise contrastive ranking on the velocity field.

\subsection{Sim-to-Real Transfer for Manipulation Policies}
\label{sub:sim_to_real}

Bridging the sim-to-real gap is essential for scaling manipulation policies beyond the small-N real-data regime, and existing approaches differ in how simulation is leveraged. \textbf{\textit{(\romannumeral1) Distribution-Matching Transfer.}} Domain randomization~\cite{tobin2017domainrandom,zhao2020simtoreal,jiang2025transic} and high-fidelity digital twins built from 3D reconstruction or generative modeling~\cite{jiang2025gsworld,wu2025rlgsbridge,xu2026twinrl,guo2026vlaw} reduce the visual and dynamics gap so that policies trained purely in simulation generalize to real hardware. \textbf{\textit{(\romannumeral2) Supervised Co-Training.}} Treats simulation as a scalable demonstration source and jointly trains on real and simulated data, consistently outperforming real-only SFT even under modest simulator fidelity~\cite{tao2024maniskill3,maddukuri2025simandrealcotraining,nasiriany2024robocasa}; tools like MimicGen~\cite{mandlekar2023mimicgen} enable few-seed synthesis of large simulated demonstration sets. \textbf{\textit{(\romannumeral3) Online Sim-Real Co-Training.}} A nascent line incorporates RL into the co-training loop, which initializes via joint SFT and runs online RL in simulation while regularizing against a real-world anchor to prevent catastrophic forgetting~\cite{shi2026rlco}. Across these families, simulation is treated either as a domain to align with or as a passive demonstration source; its capacity to expose \textit{privileged supervisory state}---ground-truth tcp, gripper and goal predicates---remains underexploited as a structured training signal. We build on the online co-training paradigm and further exploit the digital twin as a training-side label generator, while the deployed policy depends only on RGB and proprioception. Runtime failure detectors~\cite{gu2025safe,pacaud2025guardian,grislain2025ifailsense} instead judge failure from observations at deployment; DLS defines the boundary in state space and shapes it in the velocity field during training, and such detectors could replace our privileged predicates where no digital twin exists.

\section{Approach: Discover–Localize–Shape Pipeline}
\label{sec:approach}

To instantiate our view of VLA adaptation as \textbf{Failure-Boundary Learning}, our DLS pipeline proceeds in two stages. In \textit{Stage 1}, we establish a \textbf{\textit{real-grounded behavioral prior}} through supervised co-training to teach the policy nominal behavior (Sec.~\ref{sub:prior}). In \textit{Stage 2} (Fig.~\ref{fig:nft_framework}), the few-shot prior uncovers its own failure modes via online exploration, driven by \textbf{\textit{semantic progress localization}} (\textbf{SPL}) that identifies \textit{where} failures occur, and \textbf{\textit{directional boundary shaping}} (\textbf{DBS}) that systematically pushes the policy's flow dynamics away from errors and toward robust success (Sec.~\ref{sub:onlinerl}).

\subsection{Stage 1: Real-Grounded Behavioral Prior}
\label{sub:prior}

\textbf{Problem Formulation.}
We formulate the manipulation task (in both real-world and simulation) as a Partially Observable Markov Decision Process $\mathcal{M} = \langle \mathcal{S}, \mathcal{A}, \mathcal{P}, \mathcal{O}, r, \ell, \gamma \rangle$, where observations $o_i$ contain RGB images and proprioceptive states. At each environment step $i \in \{0,\dots,K{-}1\}$, the policy $\pi_\theta$ receives $o_i$ and language instruction $\ell$, and predicts an action chunk $A_{i:i+H-1} \sim \pi_\theta(\cdot \mid o_i, \ell)$. Flow-based VLAs generate $A_i$ via iterative denoising over continuous time $t \in [0,1]$, discretized into $T$ solver steps ($1{=}t_0{>}\cdots{>}t_T{=}0$), progressively refining noisy actions $\{a_{t_j}\}_{j=0}^{T}$ into the final executable chunk. The policy learns a time-dependent velocity field $v_\theta(a_t, t, o, \ell)$ that transports Gaussian noise $a_1 \sim \mathcal{N}(0, I)$ toward a target action $a_0$ along the linear interpolant $a_t = t a_1 + (1{-}t) a_0$. Given an expert dataset $\mathcal{D}$, supervised fine-tuning minimizes:
\begin{equation}
    \mathcal{L}_{\mathrm{SFT}}(\theta; \mathcal{D}) = \mathbb{E}_{(\cdot) \sim \mathcal{D}} \; \mathbb{E}_{i, t, a_1} \left[ \left\| v_\theta(a_t, t, o_i, \ell) - (a_1 - a_0) \right\|^2 \right].
    \label{eq:sft}
\end{equation}
\textbf{Simulated Data Generation.}
We construct a digital twin in ManiSkill~\cite{tao2024maniskill3}, with simulated cameras configured using real-world hand-eye calibration matrices. To scale up expert demonstrations, we apply MimicGen~\cite{mandlekar2023mimicgen} to five real-world seed trajectories. These seeds are decomposed into task-relevant segments and geometrically transformed with randomized initial object positions to synthesize novel, large-scale simulated demonstrations.

\textbf{Co-training objective.}
To anchor the policy in the deployment domain while enriching behavioral coverage for downstream exploration, we optimize a Sim-Real mixture objective~\cite{shi2026rlco} over a limited set of real demonstrations $\mathcal{D}_{\mathrm{real}}$ and simulated dataset $\mathcal{D}_{\mathrm{sim}}$, with $\alpha \in [0, 1]$ controlling the supervision ratio between the two domains.
\begin{equation}
    \mathcal{L}_{\mathrm{Sim\text{-}Real}}(\theta) = \alpha\, \mathcal{L}_{\mathrm{SFT}}(\theta; \mathcal{D}_{\mathrm{Sim}}) + (1-\alpha)\, \mathcal{L}_{\mathrm{SFT}}(\theta; \mathcal{D}_{\mathrm{Real}}),
    \label{eq:cosft}
\end{equation}

\subsection{Stage 2: Failure-Boundary Learning via Online Exploration}
\label{sub:onlinerl}

\begin{figure}[!t]
  \centering
    \includegraphics[width=\linewidth, trim=0 0 0 0, clip]{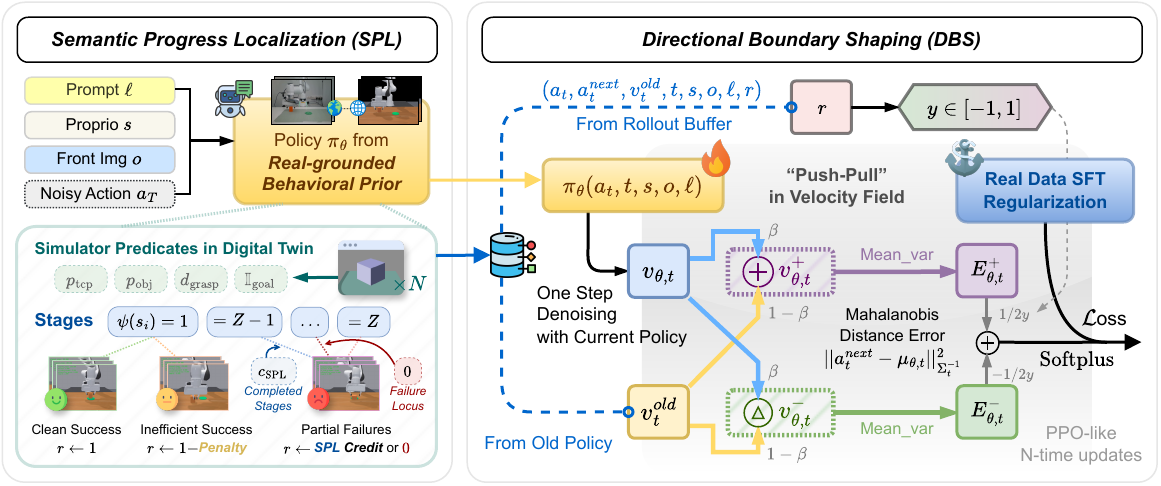}
    \caption{\textbf{Overview of DLS Pipeline.} The rollout policy from \textit{real-grounded behavioral prior} interacts with the digital twin; \textit{semantic progress localization} assigns progress-aware labels to trajectories. Two mirrored branches are constructed around the stored rollout velocity, whose Mahalanobis errors drive \textit{directional boundary shaping}, pulling the flow toward success-producing directions and pushing it away from failure-producing ones, anchored by real-data SFT regularization.}
  \label{fig:nft_framework}
  \vspace{-10pt}
\end{figure}

\textbf{On-policy Boundary Discovery.}
We deploy the prior policy $\pi_\theta^{\mathrm{old}}$ across $N_{\mathrm{env}}$ parallel simulation environments for at most $K$ steps, using the Flow-SDE sampler~\cite{chen2025pirl} to inject stochastic exploration. Each denoising step induces a Gaussian transition with mean:
\begin{equation}
    \mu_t = a_t + \left[ v_\theta(a_t,t,o,\ell) + \frac{\sigma_t^2}{2t} \big( a_t + (1{-}t)v_\theta(a_t,t,o,\ell) \big) \right] \cdot (-\delta),
    \label{eq:affine_mean}
\end{equation}
and covariance $\Sigma_t = \sigma_t^2\delta\cdot I$ for solver step size $\delta>0$. A denoising step index $t_i \sim \mathcal{U}\{0,\ldots,T{-}1\}$ is uniformly sampled per environment step $i$, and the transition tuple is stored in exploration buffer $\mathcal{B}$:
\begin{equation}
    \mathcal{B} \leftarrow \left\{ \left( a_{t_i},\; 
    a_{t_i}^{\mathrm{next}},\; v_{t_i}^{\mathrm{old}},\; 
    t_i,\; s_i,\; o_i,\; \ell,\; r_i \right) \right\}_{i=0}^{K}.
    \label{eq:buffer}
\end{equation}
where $r_i$ is the phase-mapped reward assigned by semantic progress localization (Sec.~\ref{sub:spl}).

\subsubsection{Semantic Progress Localization (SPL)}
\label{sub:spl}

Standard binary outcome labels collapse trajectory-level causal structure into a single bit, conflating qualitatively distinct failure modes---grasp failure, mid-transport drop, placement misalignment---that demand different corrective signals. Semantic Progress Localization resolves this by casting manipulation as a \textbf{hybrid state transition system}: each trajectory is mapped onto an ordered sequence of $Z$ task phases via an abstraction $\psi(s_i) \in \{1,\dots,Z\}$, where phase transitions are triggered by privileged simulator predicates over end-effector pose $p_{\mathrm{tcp}}$, object poses $p_{\mathrm{obj(s)}}$, gripper grasping state $d_{\mathrm{grasp}}\ /\  F_\mathrm{grasp}$, and goal alignment $\mathbb{I}_{\mathrm{goal}}$. Within each phase $z$, a continuous intra-phase metric $\phi_z(s_i) \in [0,1]$ measures proximity to the next phase boundary, so the discrete phase index carries continuous progress---hence \textit{hybrid}. The state potential is:
\begin{equation}
    \Phi(s_i) = W_{z-1} + w_z \cdot \phi_z(s_i),
    \label{eq:potential}
\end{equation}
where $\mathbf{w} = (w_1,\dots,w_Z)$ with $\sum_j w_j = 1$ allocates credit across phases, and $W_{z-1} = \sum_{j<z} w_j$ is the cumulative baseline for phase $z$. By projecting trajectories onto this transition system, the semantic credit and the failure-boundary-reaching phase are jointly localized as:
\begin{equation}
    c_{\mathrm{SPL}} = \max_i\,\Phi(s_i),
    \qquad
    \hat{z}_{\max} = \max_i\,\psi(s_i).
    \label{eq:spl}
\end{equation}
The cumulative baseline $W_{z-1}$ acts as a \textit{discrete gate}: credit for phase $z$ is unattainable without completing phase $z{-}1$, eliminating shortcut behaviors that exploit Euclidean proximity without task progression. The intra-phase term $w_z\phi_z$ meanwhile preserves \textit{continuous gradient signal} within each phase.
The credit $c_{\mathrm{SPL}}$ thus simultaneously serves as a \textbf{diagnostic signal}---revealing \textit{where} the competence boundary lies---and as a \textbf{stage-level training signal} preserving partial supervision for near-miss rollouts. SPL produces a stage-indexed reward map $\{r_z\}_{z=1}^{Z}$ that integrates both roles:
\begin{equation}
    r_{z=\psi(s_i)} =
    \begin{cases}
        1 - \eta\,\dfrac{i_{\mathrm{end}} - K_{\min}}{K - K_{\min}}, & c_{\mathrm{SPL}} = 1.0 \quad\text{(success)}, \\
        c_{\mathrm{SPL}}, & c_{\mathrm{SPL}} < 1.0, \ z < \hat{z}_{\max}\quad\text{(completed stage)}, \\[6pt]
        0, & c_{\mathrm{SPL}} < 1.0, \ z = \hat{z}_{\max} \quad\text{(failure locus)},
    \end{cases}
    \label{eq:reward}
\end{equation}
where $i_{\mathrm{end}}$ is the termination step, $K_{\min}$ is the minimum steps for flawless execution, and $\eta \in (0,1)$ scales the efficiency penalty. Completed stages share the trajectory credit $c_{\mathrm{SPL}}$, while the failure locus $\hat{z}_{\max}$ receives zero credit---concentrating the corrective signal at the exact stage where execution breaks down. Stages $z > \hat{z}_{\max}$ are not visited and contribute no tuples to $\mathcal{B}$. Each buffer tuple thus carries a phase-constant reward $r_i = r_{\psi(s_i)}$ and signed label $y = 2r_{\psi(s_i)} - 1$, providing the per-tuple training signal for directional boundary shaping (Sec.~\ref{sub:dbs}). SPL thereby operates as a \textit{training-side label generator} requiring no learned components or human annotation. Eq.~\eqref{eq:reward} is a phase-level \textbf{surrogate} for the recoverability label $y=\operatorname{sign}(\hat\rho_\pi-\varepsilon)$ (App.~\ref{sec:phase_gating}); App.~\ref{sub:probe} measures how closely its failure locus matches where recoverability is actually lost.

\subsubsection{Directional Boundary Shaping (DBS)}
\label{sub:dbs}

Given the phase-constant signed label $y \in [-1, 1]$ defined above, we apply directional boundary shaping directly in the flow velocity field, rather than in the action space, which would require action log-likelihoods or critic networks. For each sampled tuple in $\mathcal{B}$, the current policy produces a stochastic prediction $v_{\theta,t} = \pi_\theta(a_{t}, t, o, \ell)$ and two mirrored branches are constructed around the deviation $\Delta v_t = v_{\theta,t} - v_{t}^{\mathrm{old}}$ from the old rollout policy:
\begin{equation}
    v_{\theta,t}^{\pm} = v_{t}^{\mathrm{old}} \pm \beta\,\Delta v_t,
    \label{eq:branches}
\end{equation}
where $\beta > 0$ controls the shaping intensity. Each branch induces a transition mean $\mu_{\theta,t}^{\pm}$ via Eq.~\eqref{eq:affine_mean}, and the step-wise Mahalanobis errors against the observed next state are:
\begin{equation}
    E_{\theta,t}^{\pm} = \left\| a_{t}^{\mathrm{next}} - \mu_{\theta,t}^{\pm} \right\|_{\Sigma_t^{-1}}^{2}.
    \label{eq:errors}
\end{equation}
$\mathcal{L}_{\mathrm{DBS}}$ leverages the signed label $y$ as a ranking constraint on the two error branches:
\begin{equation}
    \mathcal{L}_{\text{DBS}}(\theta) = \mathbb{E}_{(\cdot) \sim \mathcal{B}}\left[\text{Softplus}\left(\tfrac{1}{2} y \cdot (E^+_{\theta,t} - E^-_{\theta,t})\right)\right]
    \label{eq:dbs}
\end{equation}
When $y{>}0$, gradient descent drives $E_{\theta,t}^{+} < E_{\theta,t}^{-}$, steering the velocity field \textit{toward} the success-producing direction; when $y{<}0$, the inequality reverses, pushing it \textit{away} from failure-producing directions. Softplus caps each tuple's gradient weight below one, preventing the runaway updates common in naive negative-imitation objectives. To guard against catastrophic forgetting of foundational skills~\cite{kirkpatrick2017overcoming}, DBS is regularized by the real-demonstration SFT anchor:
\begin{equation}
    \mathcal{L}_{\mathrm{total}}(\theta;\mathcal{D}_\mathrm{real}) = \mathcal{L}_{\mathrm{DBS}}(\theta) + \lambda\,\mathcal{L}_{\mathrm{SFT}}(\theta;\,\mathcal{D}_{\mathrm{real}}).
    \label{eq:total}
\end{equation}
After each iteration, the rollout policy is synchronized via EMA: $\theta^{\mathrm{old}} \leftarrow \epsilon \theta^{\mathrm{old}} + (1{-}\epsilon)\theta$, with $\epsilon$ annealed upward over training to balance early boundary discovery with late-stage stability. The exploration buffer $\mathcal{B}$ is cleared after each synchronization to maintain strictly on-policy boundary shaping and prevent stale velocities from biasing the velocity field.

\section{Experiment}
\label{sec:experiment}

We evaluate the DLS pipeline for failure-boundary learning by addressing three questions: \textbf{(\romannumeral1)} Does the discover--localize--shape pipeline yield consistent improvements over SFT baselines and alternative RL methods across tasks and backbones? \textbf{(\romannumeral2)} Is semantic progress localization the critical driver of improvement---over both binary-reward RL and alternative progress-aware reward designs? \textbf{(\romannumeral3)} How sensitive is the pipeline to key choices like regularization strength, and real data scale?

\subsection{Experimental Setup}
\label{sub:exp_settings}

\textbf{Tasks.}
We validate DLS on three real-robot manipulation tasks spanning distinct behavioral regimes and precision demands (Fig.~\ref{fig_tasks}): \textbf{Pick \& Place} (grasp from a randomized $30{\times}30$\,cm workspace with centimeter-precision placement on a coaster), \textbf{Sweep} (non-prehensile push into a $12{\times}10$\,cm target region), and \textbf{Plugin} (tight-tolerance peg-in-socket insertion with ${\sim}3$\,mm clearance).

\begin{wrapfigure}{r}{0.57\textwidth}
    \vspace{-0.4cm}
    \centering
    \includegraphics[width=\linewidth]{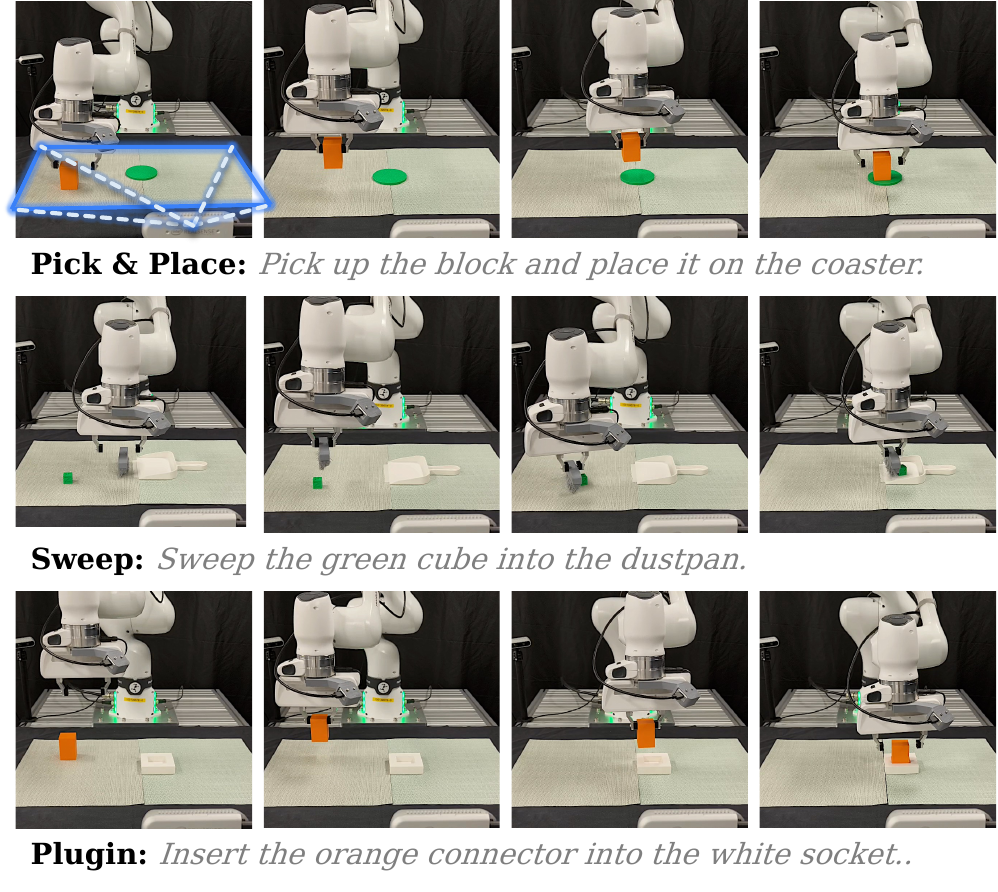}
    \vspace{-0.6cm}
    \caption{\textbf{Visualization of Real-robot Tasks Execution.}}
    \label{fig_tasks}
    \vspace{-0.5cm}
\end{wrapfigure}

\textbf{Evaluation Protocol.}
Each policy is evaluated over 30 real-robot rollouts under two conditions: \textbf{(\romannumeral1)} \textit{in-distribution} (\textit{ID.}) with 980~Lux unified lighting and standard initial position, and \textbf{(\romannumeral2)} \textit{unseen} (\textit{Unseen.}) with altered tablecloth and background, 560~Lux single-side illumination, and randomized object positions.

\textbf{Hardware \& Data.}
All experiments use a Franka Research~3 manipulator observed by a single front-facing RealSense D435 RGB camera ($640{\times}480$, $30$\,fps), 10\,Hz control, action chunk $H{=}8$; 50 real demonstrations per task collected via teleoperation. The ManiSkill~\cite{tao2024maniskill3} digital twin is geometrically calibrated from hand-eye matrices; MimicGen~\cite{mandlekar2023mimicgen} synthesizes 1,000 simulation trajectories from 5 seeds with camera perturbations (${\pm}1^\circ$, ${\pm}2$\,cm)~\cite{xie2026multi}.

\textbf{Baselines.}
We consider two groups. \textbf{(\romannumeral1)} \textit{Full-SFT}: \textbf{Real-only} fine-tunes on $N{=}50$ or $N{=}100$ real demonstrations.
\textbf{(\romannumeral2)} \textit{Few-shot SFT\,+\,RL} (all at $N{=}50$): \textbf{Sim-Real SFT} is the co-trained prior without online RL; \textbf{+\,GRPO}~\cite{liu2025flowgrpo} and \textbf{+\,PPO}~\cite{zhang2025reinflow} apply standard RL objectives---the former requiring intra-group rollout comparison, the latter a learned critic head---both absent in DLS.

\textbf{Implementation.}
We instantiate DLS on two open-source flow-matching VLAs: $\pi_0$~\cite{black2024pi_0} and $\pi_{0.5}$~\cite{intelligence2025pi_05}. \textit{Stage~1}: co-training uses Sim-Real ratio $\alpha{=}0.5$, learning rate $2.5{\times}10^{-5}$ (cosine decay, 500-step warmup); \textit{Stage~2}: explores and shapes with $N_{\mathrm{env}}{=}128$ parallel digital twins with 5 update epochs per rollout batch, learning rate $8{\times}10^{-6}$; branch scale $\beta{=}1.0$; clean-success threshold 100 steps; efficiency penalty $\eta{=}0.1$. Full hyperparameters, including how $\lambda$ is set, are in App.~\ref{sec:hyperparameters}. A rollout is counted as success upon stable task completion within $K{=}400$ environment steps; all training uses 4${\times}$NVIDIA H200 GPUs.

\subsection{Main Results}
\label{sub:main_results}

\begin{table}[t]
\centering
\caption{\textbf{Main real-robot results} across three manipulation tasks under in-distribution (\textit{ID.}) and unseen conditions (\textit{Unseen.}). Task cells: successes out of 30 real-world rollouts; Avg.: success rate (\%) over the 90 rollouts with its Wilson 95\% interval (per-cell intervals in App.~Tab.~\ref{tab:wilson}); $N$ = real demonstrations. Rewards: +GRPO and +PPO optimize the sparse binary task outcome, with group-normalized and critic-estimated advantages respectively; DLS uses SPL labels (Eq.~\eqref{eq:reward}). Within each backbone, \textbf{bold}: highest per column; \underline{underline}: second best.}
\label{tab:eval_results}
\resizebox{\textwidth}{!}{%
\small
\begin{tabular}{ll cc cc cc cc}
\toprule
\multirow{2}{*}{Model} & \multirow{2}{*}{Training Setting}
  & \multicolumn{2}{c}{Pick \& Place}
  & \multicolumn{2}{c}{Sweep}
  & \multicolumn{2}{c}{Plugin}
  & \multicolumn{2}{c}{Avg.~($\uparrow$)} \\
\cmidrule(lr){3-4}\cmidrule(lr){5-6}\cmidrule(lr){7-8}\cmidrule(lr){9-10}
 & & \textit{ID.} & \textit{Unseen.} & \textit{ID.} & \textit{Unseen.} & \textit{ID.} & \textit{Unseen.} & \textit{ID.} & \textit{Unseen.} \\
\midrule

\multirow{9}{*}{$\pi_0$}

  & \multicolumn{9}{l}{\cellcolor{lightgrey}\textit{\textcolor[RGB]{105,105,105}{\# Full SFT}}} \\

  & Real-only (\textit{N = 50})
    & 20/30 & 18/30 & 16/30 & 14/30 & 22/30 & 13/30
    & 64.4\,{\scriptsize[54.1,\,73.6]} & 50.0\,{\scriptsize[39.9,\,60.1]} \\

  & Real-only (\textit{N = 100})
    & \underline{25/30} & 19/30 & \underline{20/30} & 15/30 & \underline{24/30} & 15/30
    & 76.7\,{\scriptsize[66.9,\,84.2]} & 54.4\,{\scriptsize[44.2,\,64.3]} \\

\cmidrule(lr){2-10}
  & \multicolumn{9}{l}{\cellcolor{lightgrey}\textit{\textcolor[RGB]{105,105,105}{\# Few-shot SFT + RL}}} \\

  & Sim-Real SFT (\textit{N = 50})
    & 20/30 & 15/30 & 13/30 & 12/30 & 17/30 & 9/30
    & 55.6\,{\scriptsize[45.3,\,65.4]} & 40.0\,{\scriptsize[30.5,\,50.3]} \\

  & + GRPO~\cite{liu2025flowgrpo}
    & 22/30 & 21/30 & 18/30 & 13/30 & \underline{24/30} & 18/30
    & 71.1\,{\scriptsize[61.0,\,79.5]} & 57.8\,{\scriptsize[47.5,\,67.5]} \\

  & + PPO~\cite{zhang2025reinflow}
    & \textbf{26/30} & \underline{23/30} & \underline{20/30} & \underline{17/30} & \underline{24/30} & \underline{19/30}
    & \underline{77.8}\,{\scriptsize[68.2,\,85.1]} & \underline{65.6}\,{\scriptsize[55.3,\,74.6]} \\

  & \cellcolor{lightgreen}\textbf{DLS (ours)}
    & \cellcolor{lightgreen}\underline{25/30} & \cellcolor{lightgreen}\textbf{24/30} & \cellcolor{lightgreen}\textbf{21/30} & \cellcolor{lightgreen}\textbf{19/30} & \cellcolor{lightgreen}\textbf{25/30} & \cellcolor{lightgreen}\textbf{21/30}
    & \cellcolor{lightgreen}\textbf{78.9}\,{\scriptsize[69.4,\,86.0]} & \cellcolor{lightgreen}\textbf{71.1}\,{\scriptsize[61.0,\,79.5]} \\

\midrule

\multirow{9}{*}{$\pi_{0.5}$}

  & \multicolumn{9}{l}{\cellcolor{lightgrey}\textit{\textcolor[RGB]{105,105,105}{\# Full SFT}}} \\

  & Real-only (\textit{N = 50})
    & 22/30 & 18/30 & 17/30 & 14/30 & 22/30 & 15/30
    & 67.8\,{\scriptsize[57.6,\,76.5]} & 52.2\,{\scriptsize[42.0,\,62.2]} \\

  & Real-only (\textit{N = 100})
    & 25/30 & 19/30 & \textbf{22/30} & 14/30 & 24/30 & 16/30
    & 78.9\,{\scriptsize[69.4,\,86.0]} & 54.4\,{\scriptsize[44.2,\,64.3]} \\

\cmidrule(lr){2-10}
  & \multicolumn{9}{l}{\cellcolor{lightgrey}\textit{\textcolor[RGB]{105,105,105}{\# Few-shot SFT + RL}}} \\

  & Sim-Real SFT (\textit{N = 50})
    & 19/30 & 15/30 & 15/30 & 12/30 & 17/30 & 10/30
    & 56.7\,{\scriptsize[46.4,\,66.4]} & 41.1\,{\scriptsize[31.5,\,51.4]} \\

  & + GRPO~\cite{liu2025flowgrpo}
    & 23/30 & 21/30 & 19/30 & 17/30 & 25/30 & \underline{20/30}
    & 74.4\,{\scriptsize[64.6,\,82.3]} & 64.4\,{\scriptsize[54.1,\,73.6]} \\

  & + PPO~\cite{zhang2025reinflow}
    & \underline{26/30} & \underline{22/30} & \underline{21/30} & \underline{18/30} & \underline{26/30} & \textbf{21/30}
    & \underline{81.1}\,{\scriptsize[71.8,\,87.9]} & \underline{67.8}\,{\scriptsize[57.6,\,76.5]} \\

  & \cellcolor{lightgreen}\textbf{DLS (ours)}
    & \cellcolor{lightgreen}\textbf{27/30} & \cellcolor{lightgreen}\textbf{24/30} & \cellcolor{lightgreen}\textbf{22/30} & \cellcolor{lightgreen}\textbf{20/30} & \cellcolor{lightgreen}\textbf{27/30} & \cellcolor{lightgreen}\textbf{21/30}
    & \cellcolor{lightgreen}\textbf{84.4}\,{\scriptsize[75.6,\,90.5]} & \cellcolor{lightgreen}\textbf{72.2}\,{\scriptsize[62.2,\,80.4]} \\

\bottomrule
\end{tabular}%
}
\end{table}

Tab.~\ref{tab:eval_results} reports success rates across all tasks, models, and conditions. We surface three findings.

\textbf{(\romannumeral1) Critic-free boundary shaping outperforms, with margins widening under distribution shift.}
DLS achieves \textbf{84.4\%}\,/\,\textbf{72.2\%} (ID. / Unseen.) on $\pi_{0.5}$ and \textbf{78.9\%}\,/\,\textbf{71.1\%} on $\pi_0$, the highest average success rate across both backbones. The margin over PPO widens from in-distribution to unseen conditions: $+3.3\%{\to}+4.4\%$ on $\pi_{0.5}$ and $+1.1\%{\to}+5.6\%$ on $\pi_0$---without a critic head or intra-group rollout comparisons. Without Stage~1 grounding, direct RL collapses to near-zero success (Fig.~\ref{fig_rl_init_ablation}), confirming that a competent prior is a prerequisite for failure-boundary discovery.

\begin{figure*}[b]
\vspace{-5pt}
    \centering

    \begin{subfigure}[t]{0.33\textwidth}
        \centering
        \includegraphics[width=\linewidth,height=3.65cm]{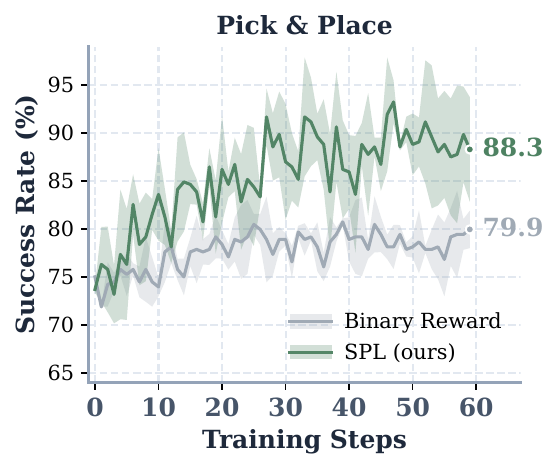}
    \end{subfigure}\hfill
    \begin{subfigure}[t]{0.33\textwidth}
        \centering
        \includegraphics[width=\linewidth,height=3.65cm]{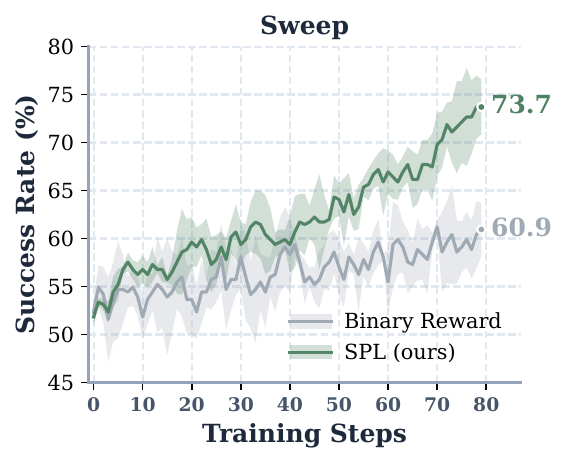}
    \end{subfigure}\hfill
    \begin{subfigure}[t]{0.33\textwidth}
        \centering
        \includegraphics[width=\linewidth,height=3.65cm]{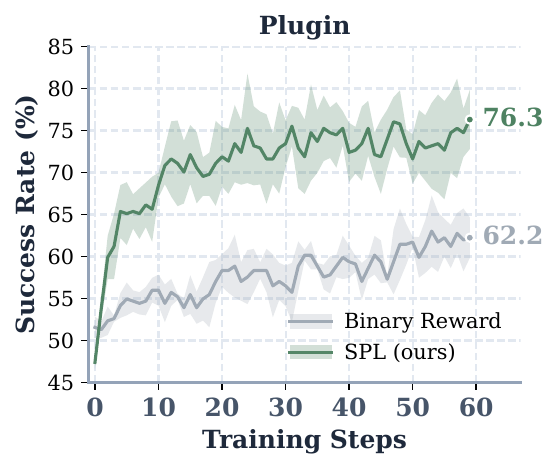}
    \end{subfigure}
    \caption{\textbf{SPL vs.\ Binary Reward: Training Curves on $\pi_{0.5}$.} \textit{Simulation success rate} over training (128 rollouts per evaluation; 40 gradient updates per step). Shaded regions denote $\pm1$ standard deviation. SPL achieves consistently faster convergence and higher final success across all tasks.}
    \label{fig:rl_curves}
\vspace{-5pt}
\end{figure*}

\textbf{(\romannumeral2) 50 real demonstrations surpass a double-budget SFT baseline.}
DLS with 50 demonstrations exceeds \textit{Real-only ($N{=}100$)} on unseen conditions across all tasks, with a \textbf{+17.8\%} absolute margin on $\pi_{0.5}$ (72.2\% vs.\ 54.4\%). Doubling real demonstrations adds only $+2.2\%$ on unseen conditions ($\pi_{0.5}$), whereas on-policy simulation rollouts surface failures that a fixed dataset does not cover. Real demonstrations serve as a grounding anchor; failure-boundary coverage requires the digital twin.

\textbf{(\romannumeral3) Distributional robustness gains exceed in-distribution gains.}
On Plugin ($\pi_{0.5}$), DLS achieves 90.0\%\,/\,70.0\% (ID.\,/\,Unseen.)---a \textbf{+33.3\%}\,/\,\textbf{+36.7\%} gain over the Sim-Real SFT prior. Unseen gains exceed in-distribution gains in 5 of 6 cases (Tab.~\ref{tab:eval_results}), suggesting that SPL-guided shaping builds structural robustness in out-of-distribution conditions rather than task-specific memorization.

\subsection{Ablation Studies}
\label{sub:ablation}

\textbf{SPL vs.\ Binary Reward.}
Fig.~\ref{fig:rl_curves} compares SPL against binary reward across tasks with random initial object position. Binary reward fails to converge reliably: treating barely-successful and efficiently-successful trajectories identically collapses task progress into a single bit. Under the same DBS optimizer, SPL achieves faster convergence and higher final performance on all tasks. SPL---pinpointing \textit{where} in task execution failure occurs---is the critical driver of improvement.

\begin{table}[t]
\centering
\caption{\textbf{SPL vs.\ alternative progress-aware designs on $\pi_{0.5}$.} Success rate (\textit{Unseen SR.}, \%, mean $\pm$ std over 3 seeds) and average episode horizon (\textit{Hor.}) across all tasks; \textbf{bold}: highest \textit{SR.} per task.}
\label{tab:dense_ablation}
\resizebox{\textwidth}{!}{%
\small
\begin{tabular}{l cc cc cc}
\toprule
\multirow{2}{*}{Method}
  & \multicolumn{2}{c}{Pick \& Place}
  & \multicolumn{2}{c}{Sweep}
  & \multicolumn{2}{c}{Plugin} \\
\cmidrule(lr){2-3}\cmidrule(lr){4-5}\cmidrule(lr){6-7}
  & \textit{Unseen SR.}~($\uparrow$) & \textit{Hor.}~($\downarrow$)
  & \textit{Unseen SR.}~($\uparrow$) & \textit{Hor.}~($\downarrow$)
  & \textit{Unseen SR.}~($\uparrow$) & \textit{Hor.}~($\downarrow$) \\
\midrule
Step-count Penalty
  & $60.0 \pm 3.3$ & 279.5
  & $53.3 \pm 3.3$ & 296.1
  & $56.7 \pm 3.3$ & 283.9 \\
Distance-based Potential
  & $66.7 \pm 10.0$ & 287.5
  & $50.0 \pm 9.0$ & 301.4
  & $63.3 \pm 9.0$ & 292.0 \\
\midrule
\rowcolor{lightgreen}
    \textbf{SPL (ours)}
      & $\mathbf{80.0} \pm 3.3$ & 280.3
      & $\mathbf{66.7} \pm 3.3$ & 300.1
      & $\mathbf{70.0} \pm 3.3$ & 287.4 \\
\bottomrule
\end{tabular}%
}
\end{table}

\textbf{SPL vs.\ Alternative Progress-aware Reward.}
Tab.~\ref{tab:dense_ablation} isolates the contribution of semantic phase localization against two alternatives on $\pi_{0.5}$: \textit{step-count penalty} discounts successful rollouts by episode length without phase decomposition; and \textit{distance-based potential} substitutes the phase abstraction $\psi$ with continuous Euclidean distance to the goal. SPL outperforms both alternatives on SR by at least $6.7\%$,
with episode length on par with the step-count penalty, confirming that semantic phase localization drives higher task success at no cost in completion time.

\begin{table}[!h]
\centering
\caption{\textbf{Boundary localization probe} ($\pi_{0.5}$, 192 failed rollouts): agreement with the true failure phase $z^\star_{\mathrm{rec}}$, Cohen's $\kappa$, median step gap $|\Delta i|$ to the crossing, and average Unseen SR (Tab.~\ref{tab:dense_ablation}).}
\label{tab:probe}
\small
\begin{tabular}{lcccc}
\toprule
Signal & Agree.\,(\%)\,$\uparrow$ & $\kappa\,\uparrow$ & $|\Delta i|\,\downarrow$ & Unseen SR\,$\uparrow$ \\
\midrule
Majority phase     & 52.1 & 0.02 & 41 & --- \\
Step-count         & 41.2 & 0.18 & 62 & 56.7 \\
Distance potential & 46.4 & 0.31 & 34 & 60.0 \\
\midrule
\rowcolor{lightgreen}
\textbf{SPL (ours)} & \textbf{78.1} & \textbf{0.68} & \textbf{19} & \textbf{72.2} \\
\bottomrule
\end{tabular}
\end{table}

\textbf{Does SPL Localize the Boundary?}
Success rates show \textit{that} SPL helps, not \textit{where} it places the failure. On 192 failed rollouts we estimate $\hat\rho_\pi$ by simulator resets and score each signal's failure phase against the true crossing $z^\star_{\mathrm{rec}}$ (App.~\ref{sub:probe}). SPL agrees in $78.1\%$ of rollouts ($\kappa{=}0.68$ vs.\ ${\le}0.31$ for the alternatives), and Unseen success ranks the same way (Tab.~\ref{tab:probe}).

\begin{figure}[!h]
\vspace{-0pt}
    \centering
    \begin{minipage}[t]{0.47\linewidth}
        \centering
        \includegraphics[width=\linewidth]{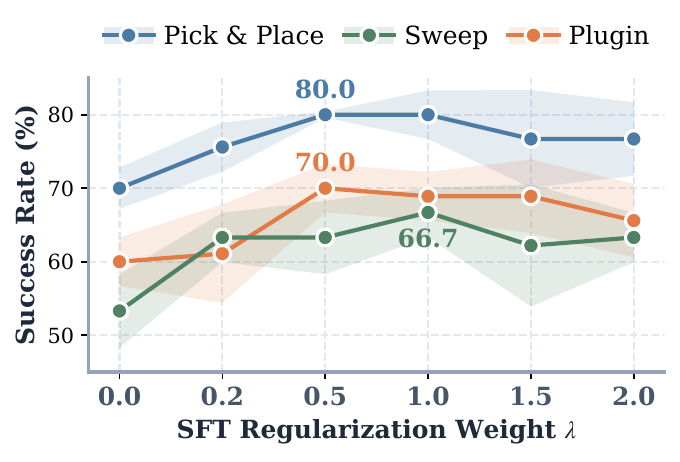}
        \caption{\textbf{Sensitivity to $\lambda$ on $\pi_{0.5}$.} Success rate peaks within $[0.5, 1.0]$ and drops most without the anchor ($\lambda{=}0$). Shaded: mean $\pm$ half the min--max range over 3 seeds.}
        \label{fig_lambda_ablation}
    \end{minipage}\hfill
    \begin{minipage}[t]{0.50\linewidth}
        \centering
        \includegraphics[width=\linewidth]{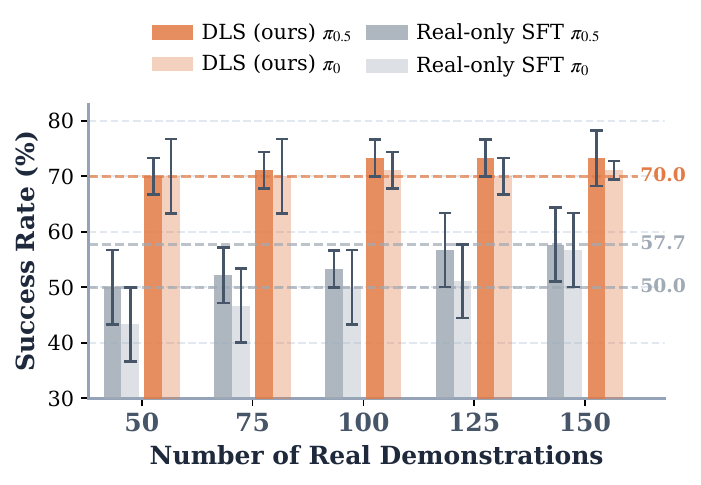}
        \caption{\textbf{Data Efficiency on Plugin.} DLS surpasses Real-only SFT with more real data budgets on both backbones. Error bars: half the min--max range over 3 seeds.}
        \label{fig_data_efficiency}
    \end{minipage}
\end{figure}

\textbf{Sensitivity to SFT Regularization Weight $\lambda$.}
Fig.~\ref{fig_lambda_ablation} sweeps $\lambda$ from $0.0$ to $2.0$ on $\pi_{0.5}$ across all tasks. The success rate (\textit{Unseen.}) follows an inverted-U: removing the anchor ($\lambda{=}0.0$) costs 10--13\% as the policy forgets the real-grounded prior; performance recovers monotonically, peaks within $\lambda \in [0.5, 1.0]$, and declines mildly under over-regularization. This validates the dual-objective design of $\mathcal{L}_{\mathrm{total}}$ in Eq.~\eqref{eq:total}: the SFT anchor preserves foundational skills.

\textbf{Data Efficiency under the Few-shot Regime.}
Fig.~\ref{fig_data_efficiency} ablates real demonstration count on Plugin across both backbones. DLS surpasses Real-only SFT success rate (Unseen.) at all demonstration counts, confirming that on-policy failure discovery in simulation compensates for sparse real-world coverage. Real demonstrations define the deployment domain; the digital twin provides the exploration scale needed to discover failure boundaries no fixed dataset can anticipate.

\vspace{-5pt}
\section{Limitation}
\label{limitation}

DLS assumes a digital twin of sufficient physical fidelity. Residual sim-to-real gaps may attenuate SPL label quality, particularly on tasks with tight contact tolerances. The current pipeline requires pre-existing simulation assets for each task object; extending to novel objects would demand 3D scanning or reconstruction, adding engineering overhead that limits out-of-the-box applicability. DBS also requires a velocity field and an SDE sampler, whose Gaussian transitions supply both exploration noise and the errors in Eq.~\eqref{eq:errors}; flow and diffusion action experts (\eg GR00T~\cite{bjorck2025gr00t}) satisfy this without change, whereas autoregressive VLAs are excluded structurally rather than by choice. Finally, SPL phases and gate predicates are specified by hand per task (App.~\ref{sec:spl_specs}).

\vspace{-5pt}
\section{Conclusion}
\label{sec:conclusion}

We propose \textbf{DLS}, a pipeline that \textit{discovers} \textbf{failure boundaries} via on-policy digital twin rollouts, \textit{localizes} them through \textbf{\textit{SPL}}'s progress-aware labels, and \textit{shapes} the flow velocity field with \textbf{\textit{DBS}}---without likelihoods or critics. Across three real-robot tasks and two VLA backbones, DLS achieves \textbf{84.4\% / 72.2\%} (ID. / Unseen.) on $\pi_{0.5}$ and \textbf{78.9\% / 71.1\%} on $\pi_0$, surpassing Real-only SFT trained with $N{=}100$ demonstrations by \textbf{+17.8\%} on unseen conditions with half the real data---showing that the policy's own failure boundary, made legible by a digital twin's privileged state, is a more data-efficient signal for generalization than additional expert demonstrations.

\acknowledgments{This research is supported by the National Research Foundation, Singapore under its AI Singapore Programme (AISG Award No: AISG3-RP-2022-030).}

\bibliography{references}  

@article{black2024pi_0,
  title={{$\pi_0$}: A Vision-Language-Action Flow Model for General Robot Control},
  author={Black, Kevin and Brown, Noah and Driess, Danny and Esmail, Adnan and Equi, Michael and Finn, Chelsea and Fusai, Niccolo and Groom, Lachy and Hausman, Karol and Ichter, Brian and others},
  journal={arXiv preprint arXiv:2410.24164},
  year={2024}
}

@article{intelligence2025pi_05,
  title={{$\pi_{0.5}$}: a vision-language-action model with open-world generalization},
  author={Intelligence, Physical and Black, Kevin and Brown, Noah and Darpinian, James and Dhabalia, Karan and Driess, Danny and Esmail, Adnan and Equi, Michael and Finn, Chelsea and Fusai, Niccolo and others},
  journal={arXiv preprint arXiv:2504.16054},
  year={2025}
}

@inproceedings{kalinowska2021ergodic,
  title={Ergodic imitation: Learning from what to do and what not to do},
  author={Kalinowska, Aleksandra and Prabhakar, Ahalya and Fitzsimons, Kathleen and Murphey, Todd},
  booktitle={2021 IEEE International Conference on Robotics and Automation (ICRA)},
  pages={3648--3654},
  year={2021},
  organization={IEEE}
}

@article{pan2025muchado,
  title={Much Ado About Noising: Dispelling the Myths of Generative Robotic Control},
  author={Pan, Chaoyi and Anantharaman, Giri and Huang, Nai-Chieh and Jin, Claire and Pfrommer, Daniel and Yuan, Chenyang and Permenter, Frank and Qu, Guannan and Boffi, Nicholas and Shi, Guanya and others},
  journal={arXiv preprint arXiv:2512.01809},
  year={2025}
}

@article{liu2025flowgrpo,
  title={Flow-grpo: Training flow matching models via online rl},
  author={Liu, Jie and Liu, Gongye and Liang, Jiajun and Li, Yangguang and Liu, Jiaheng and Wang, Xintao and Wan, Pengfei and Zhang, Di and Ouyang, Wanli},
  journal={arXiv preprint arXiv:2505.05470},
  year={2025}
}

@article{chen2025pirl,
  title={{$\pi_{RL}$}: Online RL Fine-tuning for Flow-based Vision-Language-Action Models},
  author={Chen, Kang and Liu, Zhihao and Zhang, Tonghe and Guo, Zhen and Xu, Si and Lin, Hao and Zang, Hongzhi and Li, Xiang and Zhang, Quanlu and Yu, Zhaofei and others},
  journal={arXiv preprint arXiv:2510.25889},
  year={2025}
}

@article{li2025grrl,
  title={Gr-rl: Going dexterous and precise for long-horizon robotic manipulation},
  author={Li, Yunfei and Ma, Xiao and Xu, Jiafeng and Cui, Yu and Cui, Zhongren and Han, Zhigang and Huang, Liqun and Kong, Tao and Liu, Yuxiao and Niu, Hao and others},
  journal={arXiv preprint arXiv:2512.01801},
  year={2025}
}

@article{intelligence2025pi06*,
  title={{$\pi_{0.6}^*$}: a VLA That Learns From Experience},
  author={Intelligence, Physical and Amin, Ali and Aniceto, Raichelle and Balakrishna, Ashwin and Black, Kevin and Conley, Ken and Connors, Grace and Darpinian, James and Dhabalia, Karan and DiCarlo, Jared and others},
  journal={arXiv preprint arXiv:2511.14759},
  year={2025}
}

@article{zheng2025diffusionnft,
  title={DiffusionNFT: Online diffusion reinforcement with forward process},
  author={Zheng, Kaiwen and Chen, Huayu and Ye, Haotian and Wang, Haoxiang and Zhang, Qinsheng and Jiang, Kai and Su, Hang and Ermon, Stefano and Zhu, Jun and Liu, Ming-Yu},
  journal={arXiv preprint arXiv:2509.16117},
  year={2025}
}

@article{wang2026pistepnft,
  title={{$\pi$}-StepNFT: Wider Space Needs Finer Steps in Online RL for Flow-based VLAs},
  author={Wang, Siting and Wang, Xiaofeng and Zhu, Zheng and Pei, Minnan and Cui, Xinyu and Deng, Cheng and Zhao, Jian and Huang, Guan and Zhang, Haifeng and Wang, Jun},
  journal={arXiv preprint arXiv:2603.02083},
  year={2026}
}

@inproceedings{gu2025safe,
  title={{SAFE}: Multitask Failure Detection for Vision-Language-Action Models},
  author={Gu, Qiao and Ju, Yuanliang and Sun, Shengxiang and Gilitschenski, Igor and Nishimura, Haruki and Itkina, Masha and Shkurti, Florian},
  booktitle={Advances in Neural Information Processing Systems},
  year={2025}
}

@article{pacaud2025guardian,
  title={Guardian: Detecting Robotic Planning and Execution Errors with Vision-Language Models},
  author={Pacaud, Paul and Garcia, Ricardo and Chen, Shizhe and Schmid, Cordelia},
  journal={arXiv preprint arXiv:2512.01946},
  year={2025}
}

@article{grislain2025ifailsense,
  title={{I-FailSense}: Towards General Robotic Failure Detection with Vision-Language Models},
  author={Grislain, Clemence and Rahimi, Hamed and Sigaud, Olivier and Chetouani, Mohamed},
  journal={arXiv preprint arXiv:2509.16072},
  year={2025}
}

@article{tao2024maniskill3,
  title={Maniskill3: Gpu parallelized robotics simulation and rendering for generalizable embodied ai},
  author={Tao, Stone and Xiang, Fanbo and Shukla, Arth and Qin, Yuzhe and Hinrichsen, Xander and Yuan, Xiaodi and Bao, Chen and Lin, Xinsong and Liu, Yulin and Chan, Tse-kai and others},
  journal={arXiv preprint arXiv:2410.00425},
  year={2024}
}

@article{mandlekar2023mimicgen,
  title={MimicGen: A data generation system for scalable robot learning using human demonstrations},
  author={Mandlekar, Ajay and Nasiriany, Soroush and Wen, Bowen and Akinola, Iretiayo and Narang, Yashraj and Fan, Linxi and Zhu, Yuke and Fox, Dieter},
  journal={arXiv preprint arXiv:2310.17596},
  year={2023}
}

@article{brohan2022rt1,
  title={Rt-1: Robotics transformer for real-world control at scale},
  author={Brohan, Anthony and Brown, Noah and Carbajal, Justice and Chebotar, Yevgen and Dabis, Joseph and Finn, Chelsea and Gopalakrishnan, Keerthana and Hausman, Karol and Herzog, Alex and Hsu, Jasmine and others},
  journal={arXiv preprint arXiv:2212.06817},
  year={2022}
}

@inproceedings{zitkovich2023rt2,
  title={Rt-2: Vision-language-action models transfer web knowledge to robotic control},
  author={Zitkovich, Brianna and Yu, Tianhe and Xu, Sichun and Xu, Peng and Xiao, Ted and Xia, Fei and Wu, Jialin and Wohlhart, Paul and Welker, Stefan and Wahid, Ayzaan and others},
  booktitle={Conference on Robot Learning},
  pages={2165--2183},
  year={2023},
  organization={PMLR}
}

@article{kim2024openvla,
  title={Openvla: An open-source vision-language-action model},
  author={Kim, Moo Jin and Pertsch, Karl and Karamcheti, Siddharth and Xiao, Ted and Balakrishna, Ashwin and Nair, Suraj and Rafailov, Rafael and Foster, Ethan and Lam, Grace and Sanketi, Pannag and others},
  journal={arXiv preprint arXiv:2406.09246},
  year={2024}
}

@article{kim2025openvlaoft,
  title={Fine-tuning vision-language-action models: Optimizing speed and success},
  author={Kim, Moo Jin and Finn, Chelsea and Liang, Percy},
  journal={arXiv preprint arXiv:2502.19645},
  year={2025}
}

@inproceedings{mees2024octo,
  title={Octo: An open-source generalist robot policy},
  author={Mees, Oier and Ghosh, Dibya and Pertsch, Karl and Black, Kevin and Walke, Homer Rich and Dasari, Sudeep and Hejna, Joey and Kreiman, Tobias and Xu, Charles and Luo, Jianlan and others},
  booktitle={Robotics: Science and Systems (RSS)},
  year={2024}
}

@article{bjorck2025gr00t,
  title={Gr00t n1: An open foundation model for generalist humanoid robots},
  author={Bjorck, Johan and Casta{\~n}eda, Fernando and Cherniadev, Nikita and Da, Xingye and Ding, Runyu and Fan, Linxi and Fang, Yu and Fox, Dieter and Hu, Fengyuan and Huang, Spencer and others},
  journal={arXiv preprint arXiv:2503.14734},
  year={2025}
}

@inproceedings{ross2011dagger,
  title={A reduction of imitation learning and structured prediction to no-regret online learning},
  author={Ross, St{\'e}phane and Gordon, Geoffrey and Bagnell, Drew},
  booktitle={Proceedings of the fourteenth international conference on artificial intelligence and statistics},
  pages={627--635},
  year={2011},
  organization={JMLR Workshop and Conference Proceedings}
}

@article{zhang2025reinflow,
  title={ReinFlow: Fine-tuning flow matching policy with online reinforcement learning},
  author={Zhang, Tonghe and Yu, Chao and Su, Sichang and Wang, Yu},
  journal={arXiv preprint arXiv:2505.22094},
  year={2025}
}

@article{prabhudesai2023alignprop,
  title={Aligning Text-to-Image Diffusion Models with Reward Backpropagation},
  author={Prabhudesai, Mihir and Goyal, Anirudh and Pathak, Deepak and Fragkiadaki, Katerina},
  journal={arXiv preprint arXiv:2310.03739},
  year={2023}
}

@article{rafailov2023dpo,
  title={Direct preference optimization: Your language model is secretly a reward model},
  author={Rafailov, Rafael and Sharma, Archit and Mitchell, Eric and Manning, Christopher D and Ermon, Stefano and Finn, Chelsea},
  journal={Advances in neural information processing systems},
  volume={36},
  pages={53728--53741},
  year={2023}
}

@inproceedings{tobin2017domainrandom,
  title={Domain randomization for transferring deep neural networks from simulation to the real world},
  author={Tobin, Josh and Fong, Rachel and Ray, Alex and Schneider, Jonas and Zaremba, Wojciech and Abbeel, Pieter},
  booktitle={2017 IEEE/RSJ international conference on intelligent robots and systems (IROS)},
  pages={23--30},
  year={2017},
  organization={IEEE}
}

@inproceedings{zhao2020simtoreal,
  title={Sim-to-real transfer in deep reinforcement learning for robotics: a survey},
  author={Zhao, Wenshuai and Queralta, Jorge Pe{\~n}a and Westerlund, Tomi},
  booktitle={2020 IEEE symposium series on computational intelligence (SSCI)},
  pages={737--744},
  year={2020},
  organization={IEEE}
}

@article{jiang2025gsworld,
  title={Gsworld: Closed-loop photo-realistic simulation suite for robotic manipulation},
  author={Jiang, Guangqi and Chang, Haoran and Qiu, Ri-Zhao and Liang, Yutong and Ji, Mazeyu and Zhu, Jiyue and Dong, Zhao and Zou, Xueyan and Wang, Xiaolong},
  journal={arXiv preprint arXiv:2510.20813},
  year={2025}
}

@inproceedings{wu2025rlgsbridge,
  title={Rl-gsbridge: 3d gaussian splatting based real2sim2real method for robotic manipulation learning},
  author={Wu, Yuxuan and Pan, Lei and Wu, Wenhua and Wang, Guangming and Miao, Yanzi and Xu, Fan and Wang, Hesheng},
  booktitle={2025 IEEE International Conference on Robotics and Automation (ICRA)},
  pages={192--198},
  year={2025},
  organization={IEEE}
}

@article{xu2026twinrl,
  title={TwinRL: Digital Twin-Driven Reinforcement Learning for Real-World Robotic Manipulation},
  author={Xu, Qinwen and Liu, Jiaming and Zhou, Rui and Shi, Shaojun and Han, Nuowei and Liu, Zhuoyang and Gu, Chenyang and Gu, Shuo and Yue, Yang and Huang, Gao and others},
  journal={arXiv preprint arXiv:2602.09023},
  year={2026}
}

@article{maddukuri2025simandrealcotraining,
  title={Sim-and-real co-training: A simple recipe for vision-based robotic manipulation},
  author={Maddukuri, Abhiram and Jiang, Zhenyu and Chen, Lawrence Yunliang and Nasiriany, Soroush and Xie, Yuqi and Fang, Yu and Huang, Wenqi and Wang, Zu and Xu, Zhenjia and Chernyadev, Nikita and others},
  journal={arXiv preprint arXiv:2503.24361},
  year={2025}
}

@inproceedings{nasiriany2024robocasa,
  title={RoboCasa: Large-Scale Simulation of Everyday Tasks for Generalist Robots},
  author={Nasiriany, Soroush and Maddukuri, Abhiram and Zhang, Lance and Parikh, Adeet and Lo, Aaron and Joshi, Abhishek and Mandlekar, Ajay and Zhu, Yuke},
  booktitle={Robotics: Science and Systems (RSS)},
  year={2024}
}

@article{shi2026rlco,
  title={Beyond Imitation: Reinforcement Learning-Based Sim-Real Co-Training for VLA Models},
  author={Shi, Liangzhi and Chen, Shuaihang and Gao, Feng and Chen, Yinuo and Chen, Kang and Zhang, Tonghe and Zang, Hongzhi and Zhang, Weinan and Yu, Chao and Wang, Yu},
  journal={arXiv preprint arXiv:2602.12628},
  year={2026}
}

@article{guo2026vlaw,
  title={Vlaw: Iterative co-improvement of vision-language-action policy and world model},
  author={Guo, Yanjiang and Lee, Tony and Shi, Lucy Xiaoyang and Chen, Jianyu and Liang, Percy and Finn, Chelsea},
  journal={arXiv preprint arXiv:2602.12063},
  year={2026}
}

@article{xie2026multi,
  title={Multi-Camera View Scaling for Data-Efficient Robot Imitation Learning},
  author={Xie, Yichen and Wang, Yixiao and Zhao, Shuqi and Wu, Cheng-En and Tomizuka, Masayoshi and Xie, Jianwen and Fang, Hao-Shu},
  journal={arXiv preprint arXiv:2604.00557},
  year={2026}
}

@article{lipman2022flow,
  title={Flow matching for generative modeling},
  author={Lipman, Yaron and Chen, Ricky TQ and Ben-Hamu, Heli and Nickel, Maximilian and Le, Matt},
  journal={arXiv preprint arXiv:2210.02747},
  year={2022}
}

@inproceedings{zhang2025grape,
  title={GRAPE: Generalizing Robot Policy via Preference Alignment},
  author={Zhang, Zijian and Zheng, Kaiyuan and Chen, Zhaorun and Jang, Joel and Li, Yi and Han, Siwei and Wang, Chaoqi and Ding, Mingyu and Fox, Dieter and Yao, Huaxiu},
  booktitle={ICRA 2025 Workshop on Foundation Models and Neuro-Symbolic AI for Robotics},
  year={2025}
}

@inproceedings{luo2024serl,
  title={Serl: A software suite for sample-efficient robotic reinforcement learning},
  author={Luo, Jianlan and Hu, Zheyuan and Xu, Charles and Tan, You Liang and Berg, Jacob and Sharma, Archit and Schaal, Stefan and Finn, Chelsea and Gupta, Abhishek and Levine, Sergey},
  booktitle={2024 IEEE International Conference on Robotics and Automation (ICRA)},
  pages={16961--16969},
  year={2024},
  organization={IEEE}
}

@article{xu2025stare,
  title={STARE-VLA: Progressive Stage-Aware Reinforcement for Fine-Tuning Vision-Language-Action Models},
  author={Xu, Feng and Zhai, Guangyao and Kong, Xin and Fu, Tingzhong and Gordon, Daniel FN and An, Xueli and Busam, Benjamin},
  journal={arXiv preprint arXiv:2512.05107},
  year={2025}
}

@article{lin2025failsafe,
  title={Failsafe: Reasoning and recovery from failures in vision-language-action models},
  author={Lin, Zijun and Duan, Jiafei and Fang, Haoquan and Fox, Dieter and Krishna, Ranjay and Tan, Cheston and Wen, Bihan},
  journal={arXiv preprint arXiv:2510.01642},
  year={2025}
}

@article{romer2026failure,
  title={Failure prediction at runtime for generative robot policies},
  author={R{\"o}mer, Ralf and Kobras, Adrian and Worbis, Luca and Schoellig, Angela},
  journal={Advances in Neural Information Processing Systems},
  volume={38},
  pages={7631--7670},
  year={2025}
}

@article{team2025gemini,
  title={Gemini robotics: Bringing ai into the physical world},
  author={Team, Gemini Robotics and Abeyruwan, Saminda and Ainslie, Joshua and Alayrac, Jean-Baptiste and Arenas, Montserrat Gonzalez and Armstrong, Travis and Balakrishna, Ashwin and Baruch, Robert and Bauza, Maria and Blokzijl, Michiel and others},
  journal={arXiv preprint arXiv:2503.20020},
  year={2025}
}

@article{xiang2025vlaposttraining,
  title={Parallels between vla model post-training and human motor learning: Progress, challenges, and trends},
  author={Xiang, Tian-Yu and Jin, Ao-Qun and Zhou, Xiao-Hu and Gui, Mei-Jiang and Xie, Xiao-Liang and Liu, Shi-Qi and Wang, Shuang-Yi and Duan, Sheng-Bin and Xie, Fu-Chao and Wang, Wen-Kai and others},
  journal={arXiv preprint arXiv:2506.20966},
  year={2025}
}

@article{kachaev2025dontblind,
  title={Don't Blind Your VLA: Aligning Visual Representations for OOD Generalization},
  author={Kachaev, Nikita and Kolosov, Mikhail and Zelezetsky, Daniil and Kovalev, Alexey K and Panov, Aleksandr I},
  journal={arXiv preprint arXiv:2510.25616},
  year={2025}
}

@inproceedings{xia2025phoenix,
  title={Phoenix: A motion-based self-reflection framework for fine-grained robotic action correction},
  author={Xia, Wenke and Feng, Ruoxuan and Wang, Dong and Hu, Di},
  booktitle={Proceedings of the IEEE/CVF Conference on Computer Vision and Pattern Recognition},
  pages={6981--6990},
  year={2025}
}

@article{welte2026flowcorrect,
  title={FlowCorrect: Efficient Interactive Correction of Generative Flow Policies for Robotic Manipulation},
  author={Welte, Edgar and Shi, Yitian and Wolf, Rosa and Gilles, Maximillian and Rayyes, Rania},
  journal={arXiv preprint arXiv:2602.22056},
  year={2026}
}

@article{zhang2025robustvla,
  title={RobustVLA: Robustness-aware reinforcement post-training for vision-language-action models},
  author={Zhang, Hongyin and Zhang, Shuo and Jin, Junxi and Zeng, Qixin and Li, Runze and Wang, Donglin},
  journal={arXiv preprint arXiv:2511.01331},
  year={2025}
}

@article{kirkpatrick2017overcoming,
  title={Overcoming catastrophic forgetting in neural networks},
  author={Kirkpatrick, James and Pascanu, Razvan and Rabinowitz, Neil and Veness, Joel and Desjardins, Guillaume and Rusu, Andrei A and Milan, Kieran and Quan, John and Ramalho, Tiago and Grabska-Barwinska, Agnieszka and others},
  journal={Proceedings of the national academy of sciences},
  volume={114},
  number={13},
  pages={3521--3526},
  year={2017},
  publisher={National Academy of Sciences}
}

@article{cen2025worldvla,
  title={Worldvla: Towards autoregressive action world model},
  author={Cen, Jun and Yu, Chaohui and Yuan, Hangjie and Jiang, Yuming and Huang, Siteng and Guo, Jiayan and Li, Xin and Song, Yibing and Luo, Hao and Wang, Fan and others},
  journal={arXiv preprint arXiv:2506.21539},
  year={2025}
}

@article{li2024cogact,
  title={Cogact: A foundational vision-language-action model for synergizing cognition and action in robotic manipulation},
  author={Li, Qixiu and Liang, Yaobo and Wang, Zeyu and Luo, Lin and Chen, Xi and Liao, Mozheng and Wei, Fangyun and Deng, Yu and Xu, Sicheng and Zhang, Yizhong and others},
  journal={arXiv preprint arXiv:2411.19650},
  year={2024}
}

@article{zhang2025dreamvla,
  title={Dreamvla: a vision-language-action model dreamed with comprehensive world knowledge},
  author={Zhang, Wenyao and Liu, Hongsi and Qi, Zekun and Wang, Yunnan and Yu, Xinqiang and Zhang, Jiazhao and Dong, Runpei and He, Jiawei and Lu, Fan and Wang, He and others},
  journal={arXiv preprint arXiv:2507.04447},
  year={2025}
}

@article{xu2025can,
  title={Can we detect failures without failure data? uncertainty-aware runtime failure detection for imitation learning policies},
  author={Xu, Chen and Nguyen, Tony Khuong and Dixon, Emma and Rodriguez, Christopher and Miller, Patrick and Lee, Robert and Shah, Paarth and Ambrus, Rares and Nishimura, Haruki and Itkina, Masha},
  journal={arXiv preprint arXiv:2503.08558},
  year={2025}
}

@article{chen2026aca,
  title={Escaping the Diversity Trap in Robotic Manipulation via Anchor-Centric Adaptation},
  author={Chen, Yanzhe and Ma, Kevin Yuchen and Lv, Qi and Lin, Yiqi and Bai, Zechen and Gao, Chen and Shou, Mike Zheng},
  journal={arXiv preprint arXiv:2605.07381},
  year={2026}
}

@article{liu2026rl4vla,
  title={What can rl bring to vla generalization? an empirical study},
  author={Liu, Jijia and Gao, Feng and Wei, Bingwen and Chen, Xinlei and Liao, Qingmin and Wu, Yi and Yu, Chao and Wang, Yu},
  journal={Advances in Neural Information Processing Systems},
  volume={38},
  pages={97121--97151},
  year={2025}
}

@article{li2025simplevlarl,
  title={Simplevla-rl: Scaling vla training via reinforcement learning},
  author={Li, Haozhan and Zuo, Yuxin and Yu, Jiale and Zhang, Yuhao and Yang, Zhaohui and Zhang, Kaiyan and Zhu, Xuekai and Zhang, Yuchen and Chen, Tianxing and Cui, Ganqu and others},
  journal={arXiv preprint arXiv:2509.09674},
  year={2025}
}

@article{shu2025rftf,
  title={Rftf: Reinforcement fine-tuning for embodied agents with temporal feedback},
  author={Shu, Junyang and Lin, Zhiwei and Wang, Yongtao},
  journal={arXiv preprint arXiv:2505.19767},
  year={2025}
}

@article{skalse2022defining,
  title={Defining and characterizing reward gaming},
  author={Skalse, Joar and Howe, Nikolaus and Krasheninnikov, Dmitrii and Krueger, David},
  journal={Advances in Neural Information Processing Systems},
  volume={35},
  pages={9460--9471},
  year={2022}
}

@article{dalal2024plan,
  title={Plan-seq-learn: Language model guided rl for solving long horizon robotics tasks},
  author={Dalal, Murtaza and Chiruvolu, Tarun and Chaplot, Devendra and Salakhutdinov, Ruslan},
  journal={arXiv preprint arXiv:2405.01534},
  year={2024}
}

@article{chen2025robohorizon,
  title={Robohorizon: An llm-assisted multi-view world model for long-horizon robotic manipulation},
  author={Chen, Zixuan and Huo, Jing and Chen, Yangtao and Gao, Yang},
  journal={arXiv preprint arXiv:2501.06605},
  year={2025}
}

@article{zhang2025reinbot,
  title={Reinbot: Amplifying robot visual-language manipulation with reinforcement learning},
  author={Zhang, Hongyin and Zhuang, Zifeng and Zhao, Han and Ding, Pengxiang and Lu, Hongchao and Wang, Donglin},
  journal={arXiv preprint arXiv:2505.07395},
  year={2025}
}

@article{chen2025sarm,
  title={SARM: Stage-Aware Reward Modeling for Long Horizon Robot Manipulation},
  author={Chen, Qianzhong and Yu, Justin and Schwager, Mac and Abbeel, Pieter and Shentu, Yide and Wu, Philipp},
  journal={arXiv preprint arXiv:2509.25358},
  year={2025}
}

@article{escoriza2025multi,
  title={Multi-Stage Manipulation with Demonstration-Augmented Reward, Policy, and World Model Learning},
  author={Escoriza, Adri{\`a} L{\'o}pez and Hansen, Nicklas and Tao, Stone and Mu, Tongzhou and Su, Hao},
  journal={arXiv preprint arXiv:2503.01837},
  year={2025}
}

@inproceedings{kim2025stage,
  title={Stage-wise reward shaping for acrobatic robots: A constrained multi-objective reinforcement learning approach},
  author={Kim, Dohyeong and Kwon, Hyeokjin and Kim, Junseok and Lee, Gunmin and Oh, Songhwai},
  booktitle={2025 IEEE International Conference on Robotics and Automation (ICRA)},
  pages={10268--10274},
  year={2025},
  organization={IEEE}
}

@inproceedings{jiang2025transic,
  title={TRANSIC: Sim-to-Real Policy Transfer by Learning from Online Correction},
  author={Jiang, Yunfan and Wang, Chen and Zhang, Ruohan and Wu, Jiajun and Fei-Fei, Li},
  booktitle={Conference on Robot Learning},
  pages={1691--1729},
  year={2025},
  organization={PMLR}
}


\appendix
\raggedbottom 

\vspace{-10pt}

\section{Implementation Details}
\label{ID}

\subsection{Construction of Digital Twin}

\textbf{Camera Matrices.}
The front-facing RealSense D435 camera undergoes hand-eye calibration. The resulting intrinsic matrix $K$ and camera-to-robot extrinsic matrix $M_{\text{cam}\to\text{robot}}$ are:
\begin{align}
    K &= \begin{bmatrix}
    f_x & 0   & c_x \\
    0   & f_y & c_y \\
    0   & 0   & 1
    \end{bmatrix}
    = \begin{bmatrix}
    607.875 & 0       & 348.961 \\
    0       & 607.719 & 270.486 \\
    0       & 0       & 1
    \end{bmatrix},
    \label{eq:intrinsic}\\[4pt]
    M_{\text{cam}\to\text{robot}}
    &= \begin{bmatrix} R & t \\ 0 & 1 \end{bmatrix}
    = \begin{bmatrix}
     0.028 &  0.217 & -0.975 &  1.100 \\
     0.999 & -0.001 &  0.028 & -0.007 \\
     0.005 & -0.975 & -0.217 &  0.258 \\
     0     &  0     &  0     &  1
    \end{bmatrix},
    \label{eq:extrinsic}
\end{align}
where $f_x, f_y$ are the focal lengths in pixels, $(c_x, c_y)$ is the principal point, $R \in \mathrm{SO}(3)$ is the rotation matrix, and $t \in \mathbb{R}^3$ is the translation vector (in meters). These parameters are directly applied to instantiate a matching camera viewpoint within the ManiSkill simulation.

\begin{figure}[htb]
  \centering
    \includegraphics[width=\linewidth, trim=3 0 0 0, clip]{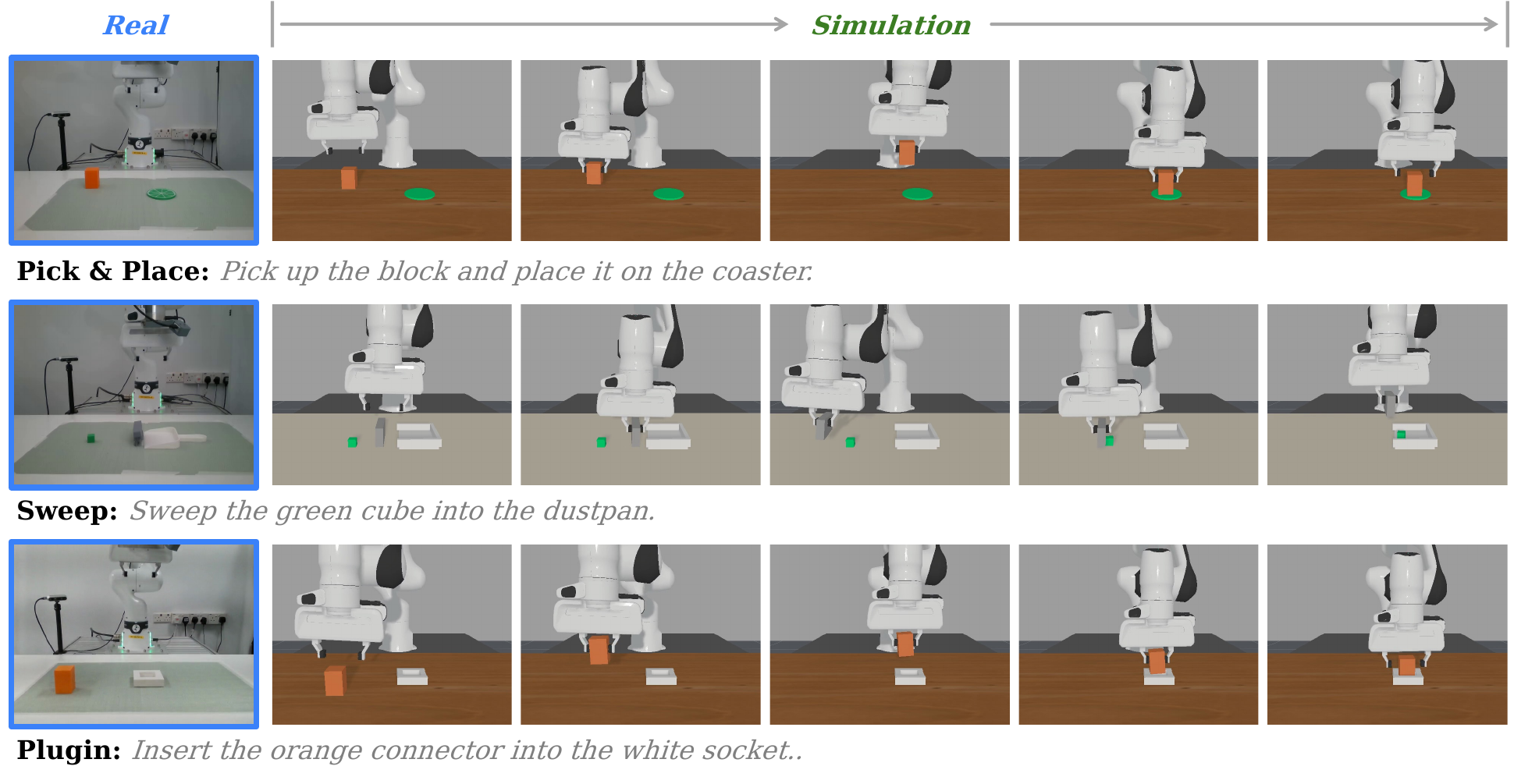}
    \caption{\textbf{Visualization of Task Execution in Digital Twin.} We visualize the tasks in constructed digital twins, showing for each task a real data reference (\textcolor[RGB]{58,129,249}{\textbf{\textit{Real}}}), a representative simulated rollout (\textcolor[RGB]{59,125,35}{\textbf{\textit{Simulation}}}), and the corresponding task instruction prompt (\textcolor[RGB]{127,127,127}{\textbf{\textit{Task Prompt}}}).
    }
  \label{fig:digitaltwin}
\end{figure}

\textbf{Visualization of Digital Twin.}
Fig.~\ref{fig:digitaltwin} provides a side-by-side comparison of the real and simulated visual observations across the three evaluated tasks. During the simulated rollouts, we introduce minor perturbations to both the camera angle (${\pm}\,1^\circ$) and its spatial coordinates (${\pm}\,0.02$\,m) to enhance the policy's robustness against inevitable physical hardware mounting variations. The strong visual alignment between the two domains confirms that the ground-truth predicates utilized by SPL—despite being extracted from privileged simulator states—remain semantically consistent with the raw visual inputs the policy receives during real-world deployment.

\subsection{SPL Phase Specifications Per Task}
\label{sec:spl_specs}

\begin{table}[t]
\centering
\caption{\textbf{SPL Phase Specifications Per Task}.
\textbf{Gate predicate}: condition at phase $z$ that triggers transition to $z{+}1$ (or task success).
Each pose $p=(xyz,\,r)$ decomposes into position $p^{xyz}_{\mathrm{obj}}\in\mathbb{R}^3$ and Euler-angle orientation $r_{\mathrm{obj}}\in\mathbb{R}^3$; superscripts on $p$ select position components (e.g.\ $p^{xy}$).
Notation: $d_z$~initial TCP--target distance at phase $z$ entry; $F_{\mathrm{grasp}}$~grasping contact force; $\mathcal{G}$~the grasped mark; $\dot{p}_{\mathrm{obj(s)}}$~object velocity;
All distances in metres; velocities in m/s; angles in degrees.}
\label{tab:spl_phases}
\renewcommand{\arraystretch}{1.4}
\resizebox{\linewidth}{!}{%
\begin{tabular}{@{}l c c l l l@{}}
\toprule
\textbf{Task} & $z$ & $w_z$ & \textbf{Phase} & \textbf{Gate Predicate} ($z \to z{+}1$ or Success) & $\phi_z(s)$ \\
\midrule
\multirow{4}{*}{\shortstack[l]{\textbf{Pick \& Place}\\[2pt]{\footnotesize $\mathrm{obj}(1)$:Block}\\{\footnotesize $\mathrm{obj}(2)$:Coaster}}}
  & 1 & 0.10 & Reach
    & $\|p_{\mathrm{tcp}}-p_{\mathrm{obj}(1)}\|\leq 0.015$
    & $1 - \|p_{\mathrm{tcp}}-p_{\mathrm{obj}(1)}\|\,/\,d_1$ \\
  & 2 & 0.25 & Grasp
    & $d_{\mathrm{grasp}}\leq 0.04\,\wedge\,F_{\mathrm{grasp}}>0$
    & $1 - |d_{\mathrm{grasp}}-0.04|\,/\,0.04$ \\
  & 3 & 0.20 & Transit
    & \textit{G}$\,\wedge\,\|p_{\mathrm{tcp}}^{xy}-p_{\mathrm{obj}(2)}^{xy}\|\leq 0.08\,\wedge\,r_{\mathrm{obj}(1)}^{\mathrm{roll/pitch}}\leq 45^{\circ}$
    & $1 - \|p_{\mathrm{tcp}}^{xy}-p_{\mathrm{obj}(2)}^{xy}\|\,/\,d_3$ \\
  & 4 & 0.45 & Place
    & $\mathbb{I}_{\mathrm{goal}}:\|p_{\mathrm{obj}(1)}^{xy}-p_{\mathrm{obj}(2)}^{xy}\|\leq 0.025\,\wedge\,r_{\mathrm{obj}(1)}^{\mathrm{roll/pitch}}\leq 5^{\circ}$
    & $1 - \|p_{\mathrm{obj}(1)}^{z}-0.025\|\,/\,d_4$ \\
\midrule
\multirow{5}{*}{\shortstack[l]{\textbf{Sweep}\\[2pt]{\footnotesize $\mathrm{obj}(1)$:Broom}\\{\footnotesize $\mathrm{obj}(2)$:Cube}\\{\footnotesize $\mathrm{obj}(3)$:Dustpan}}}
  & 1 & 0.08 & Reach
    & $\|p_{\mathrm{tcp}}^{yz}-p_{\mathrm{obj}(1)}^{yz}\|\leq 0.01\,\wedge\,\|p_{\mathrm{tcp}}^{x}-p_{\mathrm{obj}(1)}^{x}\|\leq 0.02$
    & $1 - \|p_{\mathrm{tcp}}^{xyz}-p_{\mathrm{obj}(1)}^{xyz}\|\,/\,d_1$ \\
  & 2 & 0.20 & Grasp
    & $d_{\mathrm{grasp}}\leq 0.03\,\wedge\,F_{\mathrm{grasp}}>0$
    & $1 - |d_{\mathrm{grasp}}-0.03|\,/\,0.03$ \\
  & 3 & 0.12 & Transit
    & $\mathcal{G}\,\wedge\,p_{\mathrm{tcp}}^{y}\leq p_{\mathrm{obj}(2)}^{y}-0.01\,\wedge\,p_{\mathrm{obj}(1)}^{z}\geq 0.05$
    & $1 - \|p_{\mathrm{tcp}}^{y}-p_{\mathrm{obj}(2)}^{y}+0.01\|\,/\,d_3$ \\
  & 4 & 0.15 & Contact
    & $\mathcal{G}\,\wedge\,\|\dot{p}_{\mathrm{obj}(1)}^{xy}-\dot{p}_{\mathrm{obj}(2)}^{xy}\|\leq0.002\,\wedge\,\dot{p}_{\mathrm{obj}(2)}^{xy}\cdot(p_{\mathrm{obj}(3)}^{xy}-p_{\mathrm{obj}(2)}^{xy})>0$
    & $1 - \|p_{\mathrm{obj}(1)}^{xy}-p_{\mathrm{obj}(2)}^{xy}\|\,/\,d_4$ \\
  & 5 & 0.45 & Push
    & $\mathbb{I}_{\mathrm{goal}}:\,|p_{\mathrm{obj}(2)}^{x}-p_{\mathrm{obj}(3)}^{x}|\leq 0.05\,\wedge\,|p_{\mathrm{obj}(2)}^{y}-p_{\mathrm{obj}(3)}^{y}|\leq 0.06$
    & $1 - \|p_{\mathrm{obj}(2)}^{xy}-p_{\mathrm{obj}(3)}^{xy}\|\,/\,d_5$ \\
\midrule
\multirow{4}{*}{\shortstack[l]{\textbf{Plugin}\\[2pt]{\footnotesize $\mathrm{obj}(1)$:Block}\\{\footnotesize $\mathrm{obj}(2)$:Socket}}}
  & 1 & 0.08 & Reach
    & $\|p_{\mathrm{tcp}}^{xyz}-p_{\mathrm{obj}(1)}^{xyz}\|\leq 0.015$
    & $1 - \|p_{\mathrm{tcp}}^{xyz}-p_{\mathrm{obj}(1)}^{xyz}\|\,/\,d_1$ \\
  & 2 & 0.22 & Grasp
    & $d_{\mathrm{grasp}}\leq 0.04\,\wedge\,F_{\mathrm{grasp}}>0$
    & $1 - |d_{\mathrm{grasp}}-0.04|\,/\,0.04$ \\
  & 3 & 0.15 & Align
    & $\mathcal{G}\,\wedge\,\|p_{\mathrm{tcp}}^{xy}-p_{\mathrm{obj}(2)}^{xy}\|\leq 0.01\,\wedge\,r_{\mathrm{obj}(1)}\leq 5^{\circ}$
    & $1 - \|p_{\mathrm{tcp}}^{xy}-p_{\mathrm{obj}(2)}^{xy}\|\,/\,d_3$ \\
  & 4 & 0.55 & Insert
    & $\mathbb{I}_{\mathrm{goal}}:\|p_{\mathrm{obj}(1)}^{xy}-p_{\mathrm{obj}(2)}^{xy}\|\leq 0.003\,\wedge\,p_{\mathrm{obj}(1)}^{z}< 0.05$
    & $1 - \|p_{\mathrm{obj}(1)}^{z}-0.003\|\,/\,d_4$ \\
\bottomrule
\end{tabular}%
}
\end{table}

\textbf{Task Phases.}
Tab.~\ref{tab:spl_phases} details the phase decomposition, gate predicates, intra-phase progress metrics $\phi_z$, and credit weights $w_z$ for each task.
Pick\,\&\,Place and Plugin each decompose into $Z{=}4$ phases; Sweep uses $Z{=}5$, with an additional \textit{Contact} stage between \textit{Transit} and \textit{Push} that verifies broom--cube velocity alignment and sweep direction before the final push.
Phase $z$ is completed, and phase $z{+}1$ entered, when its gate predicate $G_z$ evaluates to \texttt{true} on the privileged simulator state; within phase $z$, $\phi_z(s_i)\in[0,1]$ measures continuous progress toward the next gate boundary. Following the SPL formulation, these signals combine into the stage-weighted potential $\Phi(s_i) = W_{z-1} + w_z \cdot \phi_z(s_i)$ for $z = \psi(s_i)$, where $W_{z-1} = \sum_{j<z} w_j$ is the cumulative baseline for phase $z$ and $\sum_z w_z = 1$.

\textbf{Weight design.}
The \textit{Grasp} phase is assigned the highest weight among non-final phases, as successful grasping is a crucial prerequisite for all subsequent stages.
\textit{Near-miss}---trajectories that enter the final phase $Z$ but fail to complete it---attain partial progress $\phi_Z^*\in[0,1)$ and thus accumulate less than the full final-phase credit.
With $w_Z{=}0.45$ (Pick\,\&\,Place and Sweep) and $w_Z{=}0.55$ (Plugin), near-miss trajectories typically score below the worst-case successful rollout in practice: $W_{Z-1}+w_Z\phi_Z^* < 1{-}\eta$ for $\eta{=}0.1$, which holds whenever $\phi_Z^*\lesssim 0.78$ (Pick\,\&\,Place and Sweep) and $\phi_Z^*\lesssim 0.82$ (Plugin).

\subsection{Training Dynamics: Label Distribution over Training}

\begin{figure*}[t]
    \centering
    \begin{subfigure}[t]{0.33\textwidth}
        \centering
        \includegraphics[width=\linewidth,height=3.65cm]{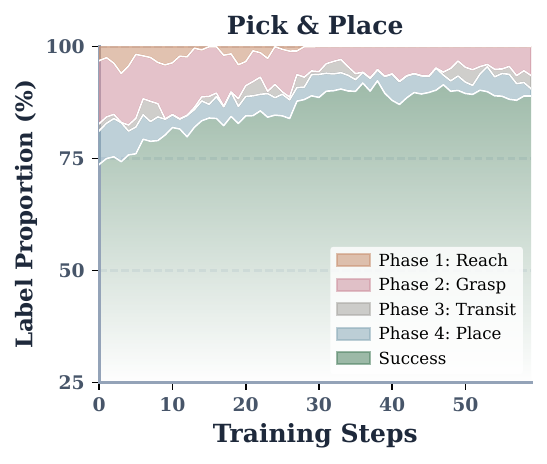}
    \end{subfigure}\hfill
    \begin{subfigure}[t]{0.33\textwidth}
        \centering
        \includegraphics[width=\linewidth,height=3.65cm]{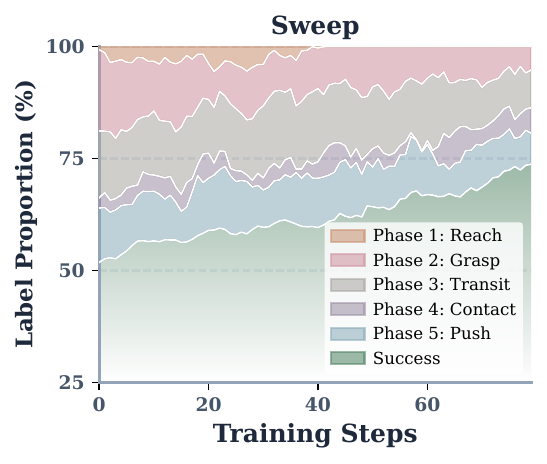}
    \end{subfigure}\hfill
    \begin{subfigure}[t]{0.33\textwidth}
        \centering
        \includegraphics[width=\linewidth,height=3.65cm]{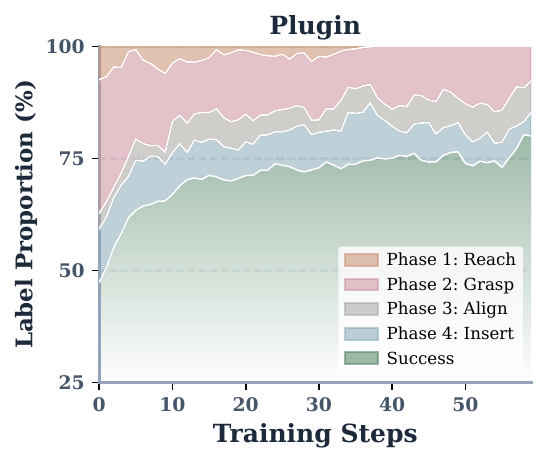}
    \end{subfigure}
    \caption{\textbf{SPL Label Distribution over Training.}
    Each stacked area shows the fraction of buffer tuples labeled at each SPL phase, plotted against training step.
    The \textcolor[RGB]{78,130,98}{\textbf{green}} band (bottom) captures successful trajectories; its height directly tracks the simulation success rate (SR).
    \textcolor[RGB]{201,144,112}{\textbf{Orange}}, 
    \textcolor[RGB]{206,145,159}{\textbf{pink}}, \textcolor[RGB]{168,166,162}{\textbf{gray}}, and \textcolor[RGB]{138,170,184}{\textbf{blue}} correspond to failure-locus phases~1, 2, 3, and $Z$;
    \textcolor[RGB]{155,142,164}{\textbf{lavender}} (Sweep only) denotes phase~4 (\textit{Contact}) failure.
    All proportions sum to 100\% at every step.
    }
    \label{fig:label_dist}
\vspace{-10pt}
\end{figure*}

The stacked distributions in Fig.~\ref{fig:label_dist} show the progression of \textbf{failure-boundary migration} throughout DLS training.
At the onset, the buffer is dominated by \textit{Grasp}-phase failure labels across all three tasks, while phase~1 (\textit{Reach}) contracts to near-zero before the midpoint.
As DBS suppresses these failure-producing velocity directions, the active locus migrates toward each task's terminal sub-goal---\textit{Place}, \textit{Insert}, and \textit{Push}---while Sweep additionally sustains a persistent \textit{Transit} component, reflecting the broom--cube alignment prerequisite that must be resolved before the push boundary can be addressed.
Throughout this migration, the \textcolor[RGB]{78,130,98}{\textbf{success}} band expands in lockstep with each resolved failure locus, corroborating that SPL's stage-indexed credit concentrates gradient signal at the active boundary rather than applying a diffuse, undirected push across all denoising steps.

\subsection{Hyperparameter}
\label{sec:hyperparameters}

\begin{table}[htp]
  \centering
  \caption{\textbf{Training Hyperparameters.}
  Parameters are grouped into three blocks: shared model configuration (\textit{Model Config}), stage~1 (\textit{Sim-Real SFT}), and stage~2 (\textit{DBS}).}
  \label{tab:hyperparameters}
  \small
  \begin{tabular}{lccc}
    \toprule
    \textbf{Parameter} & \textbf{Pick\,\&\,Place} & \textbf{Sweep} & \textbf{Plugin} \\
    \midrule
    \rowcolor{lightgrey}\multicolumn{4}{l}{\textcolor[RGB]{105,105,105}{\textit{Model Config}}} \\
    Denoising steps $T$          & 5             & 5             & 5 \\
    Action chunk / horizon       & 8 / 8         & 8 / 12        & 8 / 8 \\
    Control frequency            & 10\,Hz (stride 2)        & 10\,Hz (stride 2)        & 10\,Hz (stride 2) \\
    \midrule
    \rowcolor{lightgrey}\multicolumn{4}{l}{\textcolor[RGB]{105,105,105}{\textit{Sim-Real SFT}}} \\
    Base model                   & \texttt{lerobot/pi05\_base} & \texttt{lerobot/pi05\_base} & \texttt{lerobot/pi05\_base} \\
    Global / Micro batch         & 128 / 32      & 128 / 32      & 128 / 32 \\
    Learning rate                & $2.5\times10^{-5}$ & $2.5\times10^{-5}$ & $2.5\times10^{-5}$ \\
    Warmup steps                 & 500       & 500       & 500 \\
    Lr schedule                  & Cosine        & Cosine        & Cosine \\
    Co-training ratio $\alpha$   & 0.5           & 0.5           & 0.5 \\
    Total training steps         & 15{,}000      & 20{,}000      & 15{,}000 \\
    \midrule
    \rowcolor{lightgrey}\multicolumn{4}{l}{\textcolor[RGB]{105,105,105}{\textit{DBS}}} \\
    Checkpoint model                   & 5k-step SFT & 7.5k-step SFT & 7.5k-step SFT \\
    Global / Micro batch         & 6400 / 320    & 6400 / 320    & 6400 / 320 \\
    Learning rate                & $8\times10^{-6}$ & $8\times10^{-6}$ & $8\times10^{-6}$ \\   
    Noise level                  & 0.2           & 0.2           & 0.2 \\
    Parallel environments        & 128           & 128           & 128 \\
    Rollout epochs               & 8             & 8             & 8 \\
    Max interact steps           & 400           & 400           & 400 \\
    Clean success threshold      & 100           & 100           & 100 \\
    Clip ratio (high)            & 1.0           & 1.0           & 1.0 \\
    Update epochs                & 5             & 5             & 5 \\
    EMA decay $\epsilon$             & $0.1\!\to\!0.995$ & $0.1\!\to\!0.995$ & $0.1\!\to\!0.995$ \\
    SFT regularization $\lambda$ & 0.5           & 1.0           & 0.5 \\
    Num of real data             & 50            & 50            & 50 \\
    Total training steps         & 60            & 80            & 60 \\
    \bottomrule
  \end{tabular}
\end{table}

Tab.~\ref{tab:hyperparameters} lists the hyperparameters for both training stages across all three tasks.
One implementation detail concerns control frequency alignment between the two stages.
Real-world demonstrations are recorded at 20\,Hz; to improve the signal-to-noise ratio of learned actions, we apply a temporal stride of~2 during stage~1 training, so the policy effectively learns action transitions at 10\,Hz.
The stage~2 digital-twin simulation is configured at the same 10\,Hz control frequency, ensuring that the flow velocity field acquired from real data operates without frequency mismatch during online exploration.
The DBS stage uses a global / micro batch of 6400 / 320 with 5 update epochs per rollout, yielding exactly 40 optimizer steps per training step.
Stage~2 runs 60--80 such training steps ($2{,}400$--$3{,}200$ updates in total), stopping where simulated success plateaus.
The SFT-anchor weight $\lambda$ is set to the ratio of Stage-1 to Stage-2 loss scales and refined by a 10-iteration probe (sensitivity in Fig.~\ref{fig_lambda_ablation}).

\section{Additional Experiments}
\label{AE}

\subsection{Qualitative Analysis on Real-Robot Rollouts}

\begin{figure*}[t]
  \centering
    \includegraphics[width=\linewidth, trim=8 0 0 0, clip]{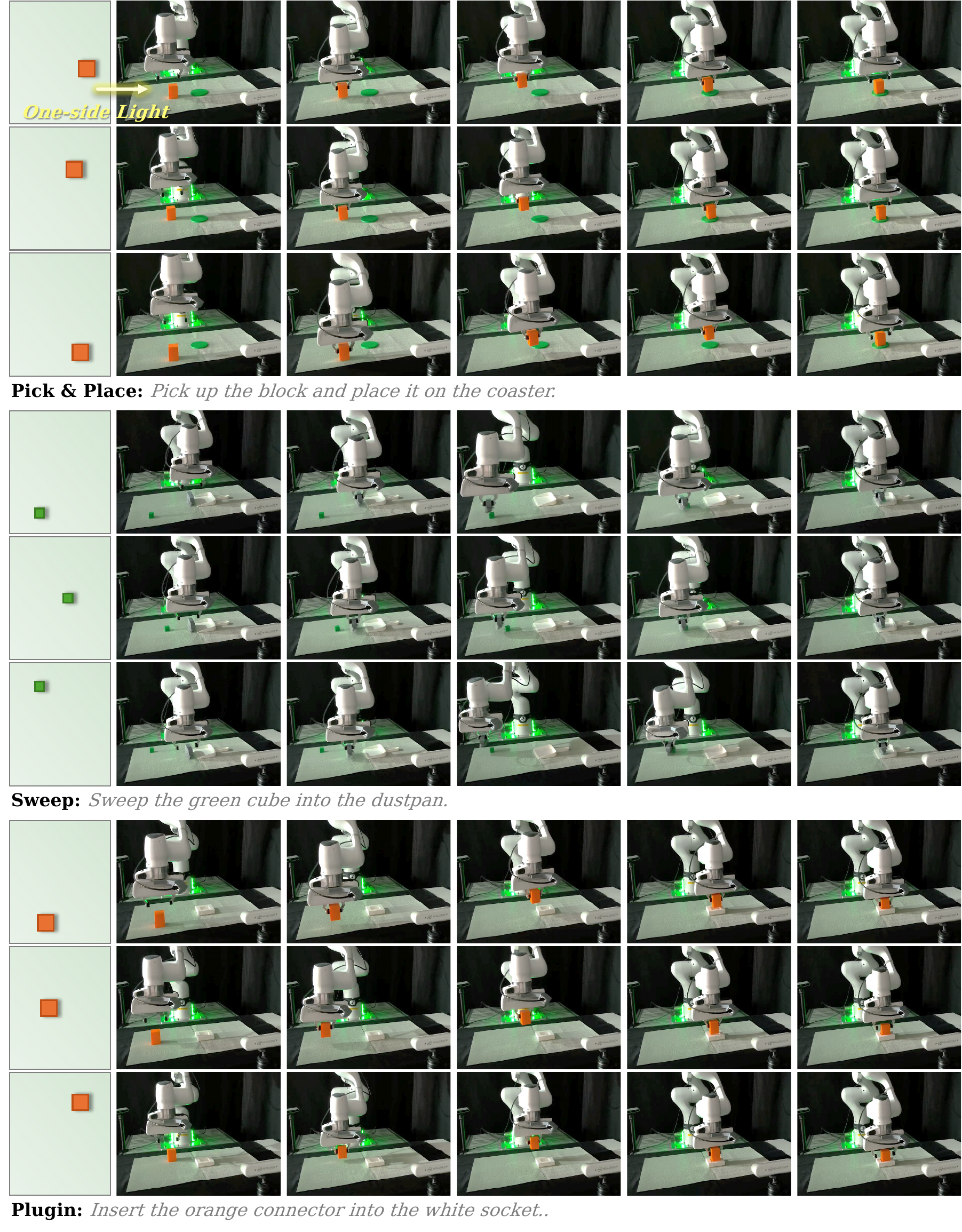}
    \caption{\textbf{Visualization of Real-robot Rollouts.}
    Each row shows one task with three representative trials.
    The leftmost column of each trial depicts the initial object position; subsequent columns show the unfolding rollout, capturing key interaction stages through to task completion.
    }
  \label{fig:tasks_real}
  \vspace{-10pt}
\end{figure*}

Fig.~\ref{fig:tasks_real} visualizes representative real-robot rollouts under the \textit{Unseen.}\ protocol: altered tablecloth and background, 560\,Lux single-side illumination (compare to 980\,Lux in real data), and fully randomized initial object positions. Three behavioral patterns emerge consistently across all three tasks.

\textbf{(\romannumeral1) Maintained phase ordering under visual distribution shift.}
Across Pick\,\&\,Place, Sweep, and Plugin, the policy executes each task phase in the correct order---reach, contact or grasp, transport or push, and final placement---even when the object starts at the extremes of the randomized workspace. Intermediate breakdowns such as approaching without establishing contact, or committing to placement prematurely, are largely absent. This is consistent with SPL's design: by concentrating gradient signal at the active failure locus, DBS shapes each phase boundary in task order rather than applying diffuse reward signal across all denoising steps.

\textbf{(\romannumeral2) Spatial generalization across the full initialization range.}
In Pick\,\&\,Place the end-effector reaches objects across the $30{\times}30$\,cm workspace without the systematic approach bias that characterizes SFT-only policies; in Sweep, the broom--cube contact trajectory adapts to varied cube starting offsets while sustaining push direction toward the goal region; in Plugin, orientation alignment precedes insertion even when the plug begins far from the nominal training pose. Rather than concentrating successes within a narrow high-confidence corridor, the policy distributes successful rollouts across the full initialization space---evidence that DLS has expanded the competence boundary beyond the region covered by the 50 real demonstrations.

\textbf{(\romannumeral3) Preserved task-critical precision under tight completion criteria.}
For all tasks, the final interaction---centimeter-level placement on the coaster, pushing the cube into the $12{\times}10$\,cm target region, and seating the plug within ${\sim}3$\,mm clearance---remains precise under unseen conditions and does not degrade to coarse near-miss behavior. This robustness is largely absent in the SFT prior, where near-miss failures at the final phase are the dominant failure mode under distribution shift, suggesting that SPL-guided boundary shaping specifically suppresses this terminal failure mode.

\subsection{Failure Mode Analysis}

Fig.~\ref{fig:failure_mode} shows representative phase-level failures of the SFT prior under the \textit{Unseen.} protocol; altered lighting and background introduce positional offsets that violate different SPL phase boundaries.

\begin{figure}[htb]
  \centering
  \includegraphics[width=\linewidth]{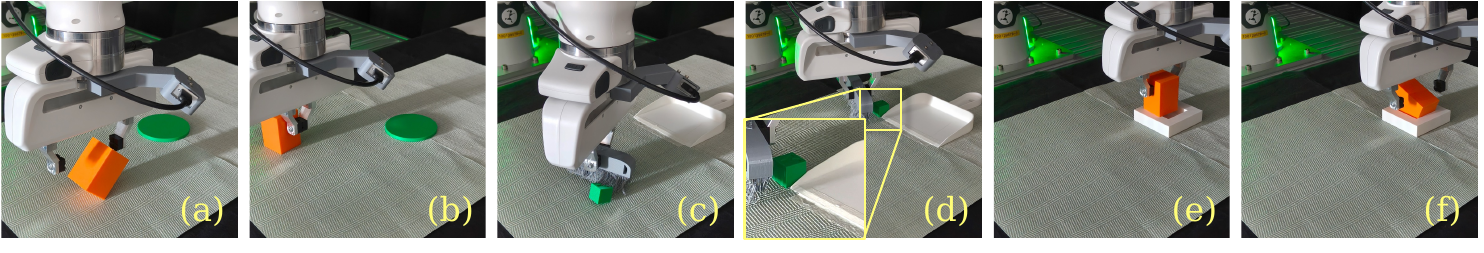}
  \caption{\textbf{SFT Prior Failure Modes under Unseen Conditions.}
  (a)--(b):~\textit{Grasp}-phase failure from lighting-induced approach offset.
  (c):~pre-contact positioning failure before \textit{Contact} phase.
  (d):~terminal push direction misalignment at \textit{Push} phase.
  (e)--(f):~terminal near-miss at \textit{Insert} phase.}
  \label{fig:failure_mode}
\end{figure}

\textbf{\textit{(\romannumeral1) Grasp-phase failure}} [(a)--(b)].
Lighting-induced lateral and depth offsets prevent the end-effector from satisfying the \textit{Grasp} gate predicate; execution stalls at an early phase boundary and never reaches \textit{Transit}.
\textbf{\textit{(\romannumeral2) Pre-terminal positioning failure}} [(c)].
The end-effector fails to reach the required lateral position relative to the cube before the \textit{Contact} phase, blocking broom--cube engagement and halting progress at an intermediate boundary.
\textbf{\textit{(\romannumeral3) Terminal near-miss}} [(d)--(f)].
In (d), broom--cube contact is established but the push direction deviates from the dustpan; in (e)--(f), the plug aligns over the socket but positional offset violates the sub-$3$\,mm \textit{Insert} tolerance. Both enter phase~$Z$ but fail to satisfy its gate---the highest-credit near-miss regime in SPL.

\textbf{Improvements after full DLS training.}
DBS concentrates repulsive gradient at each SPL-localized failure locus, suppressing lighting-induced positional drift across all phase boundaries; the dominant failure mechanism under the \textit{Unseen.} protocol is effectively eliminated.
Residual failures occur at positions where the single RGB camera cannot resolve fore-aft depth, a sensor-level ambiguity outside the failure manifold reachable through simulation exploration.

\subsection{Additional Ablation}

\textbf{Ablation on Behavioral Prior.}
Fig.~\ref{fig_rl_init_ablation} isolates the stage~1 prior by fixing stage~2 and varying data composition on Plugin ($\pi_{0.5}$): all three non-full priors underperform the Sim-Real SFT prior used by DLS ($70.0\%$ Unseen. SR after Stage~2), each for a distinct reason.
\textit{Without SFT} collapses to near-zero, as undirected rollouts yield no contrast for DBS; \textit{Real-only SFT} ($46.7\%$) confines exploration to the 50-demo region, leaving failure boundaries beyond it undiscovered; \textit{Sim-only SFT} ($60.0\%$) gains broad coverage but loses real-world grounding, misaligning simulated failure loci with those at deployment.
Simulation coverage and real-world grounding are thus orthogonal prerequisites: neither substitutes for the other.

\begin{figure}[htp]
    \centering
    \begin{minipage}[t]{0.47\linewidth}
        \centering
        \includegraphics[width=\linewidth]{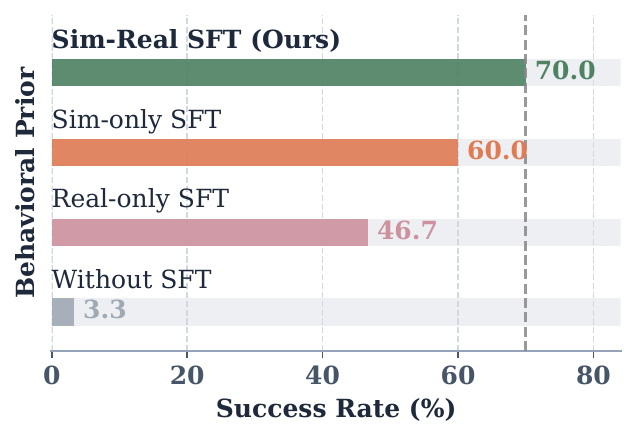}
        \caption{\textbf{Behavioral Prior Ablation on Plugin ($\pi_{0.5}$).} Unseen SR after Stage~2 from each prior; all non-full priors underperform the Sim-Real SFT prior (Ours, dashed); failure traces to missing coverage, missing grounding, or both.}
        \label{fig_rl_init_ablation}
    \end{minipage}\hfill
    \begin{minipage}[t]{0.50\linewidth}
        \centering
        \includegraphics[width=\linewidth]{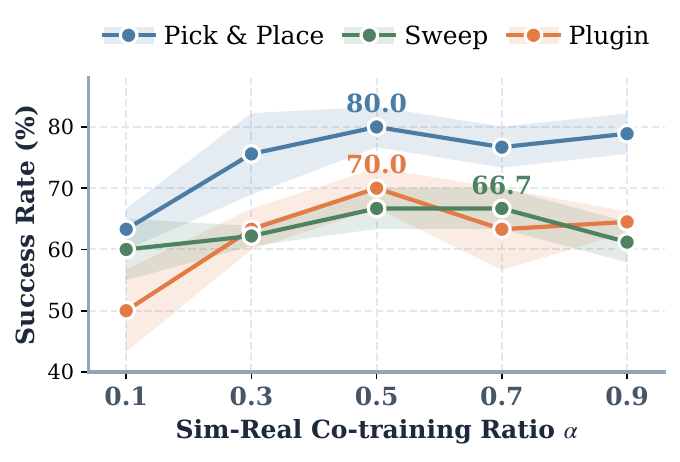}
        \caption{\textbf{Sensitivity to Sim-Real Ratio $\alpha$ on $\pi_{0.5}$.} Performance peaks near $\alpha{=}0.5$; real-dominated degrades more sharply than sim-dominated. Shaded: mean $\pm$ half the min--max range over 3 seeds.}
        \label{fig_sim_real_ratio}
    \end{minipage}
\end{figure}

\textbf{Sensitivity to Sim-Real Co-training Ratio $\alpha$.}
Fig.~\ref{fig_sim_real_ratio} sweeps $\alpha \in \{0.1, 0.3, 0.5, 0.7, 0.9\}$ on $\pi_{0.5}$ with 50 real and 1,000 simulated demonstrations fixed; Performance peaks at the median across all tasks.
The two tails degrade asymmetrically: real-dominated ($\alpha \to 0.1$) narrows stage~2 exploration coverage, suppressing DBS gradient contrast and lowering the success-rate ceiling; sim-dominated ($\alpha \to 0.9$) retains broad exploration diversity and degrades only moderately, as stage~2's SFT regularization $\lambda\,\mathcal{L}_{\mathrm{SFT}}(\mathcal{D}_{\mathrm{real}})$ re-anchors the policy to the deployment domain.
This asymmetry confirms that exploration coverage is the harder constraint to satisfy post-hoc, while insufficient real grounding is partially recoverable through regularization.

\subsection{Statistical Reporting of Main Results}
\label{sub:wilson}
Tab.~\ref{tab:wilson} lists the Wilson 95\% interval of every cell in Tab.~\ref{tab:eval_results} (30 rollouts per task cell, 90 per average). With 30 trials a single success moves a cell by 3.3\,pp, so we read differences through the averages and their consistency across tasks and backbones rather than through individual cells.

\begin{table}[h]
\centering
\caption{\textbf{Wilson 95\% intervals (\%)} for Tab.~\ref{tab:eval_results}. Task cells: $n{=}30$; averages: $n{=}90$.}
\label{tab:wilson}
\resizebox{\textwidth}{!}{%
\begin{tabular}{l cc cc cc cc}
\toprule
\multirow{2}{*}{Training Setting} & \multicolumn{2}{c}{Pick \& Place} & \multicolumn{2}{c}{Sweep} & \multicolumn{2}{c}{Plugin} & \multicolumn{2}{c}{Avg.} \\
\cmidrule(lr){2-3}\cmidrule(lr){4-5}\cmidrule(lr){6-7}\cmidrule(lr){8-9}
 & \textit{ID.} & \textit{Unseen.} & \textit{ID.} & \textit{Unseen.} & \textit{ID.} & \textit{Unseen.} & \textit{ID.} & \textit{Unseen.} \\
\multicolumn{9}{l}{\textit{$\pi_0$}} \\
Real-only (\textit{N = 50}) & 48.8--80.8 & 42.3--75.4 & 36.1--69.8 & 30.2--63.9 & 55.6--85.8 & 27.4--60.8 & 54.1--73.6 & 39.9--60.1 \\
Real-only (\textit{N = 100}) & 66.4--92.7 & 45.5--78.1 & 48.8--80.8 & 33.2--66.8 & 62.7--90.5 & 33.2--66.8 & 66.9--84.2 & 44.2--64.3 \\
Sim-Real SFT (\textit{N = 50}) & 48.8--80.8 & 33.2--66.8 & 27.4--60.8 & 24.6--57.7 & 39.2--72.6 & 16.7--47.9 & 45.3--65.4 & 30.5--50.3 \\
+ GRPO~\cite{liu2025flowgrpo} & 55.6--85.8 & 52.1--83.3 & 42.3--75.4 & 27.4--60.8 & 62.7--90.5 & 42.3--75.4 & 61.0--79.5 & 47.5--67.5 \\
+ PPO~\cite{zhang2025reinflow} & 70.3--94.7 & 59.1--88.2 & 48.8--80.8 & 39.2--72.6 & 62.7--90.5 & 45.5--78.1 & 68.2--85.1 & 55.3--74.6 \\
\textbf{DLS (ours)} & 66.4--92.7 & 62.7--90.5 & 52.1--83.3 & 45.5--78.1 & 66.4--92.7 & 52.1--83.3 & 69.4--86.0 & 61.0--79.5 \\
\midrule
\multicolumn{9}{l}{\textit{$\pi_{0.5}$}} \\
Real-only (\textit{N = 50}) & 55.6--85.8 & 42.3--75.4 & 39.2--72.6 & 30.2--63.9 & 55.6--85.8 & 33.2--66.8 & 57.6--76.5 & 42.0--62.2 \\
Real-only (\textit{N = 100}) & 66.4--92.7 & 45.5--78.1 & 55.6--85.8 & 30.2--63.9 & 62.7--90.5 & 36.1--69.8 & 69.4--86.0 & 44.2--64.3 \\
Sim-Real SFT (\textit{N = 50}) & 45.5--78.1 & 33.2--66.8 & 33.2--66.8 & 24.6--57.7 & 39.2--72.6 & 19.2--51.2 & 46.4--66.4 & 31.5--51.4 \\
+ GRPO~\cite{liu2025flowgrpo} & 59.1--88.2 & 52.1--83.3 & 45.5--78.1 & 39.2--72.6 & 66.4--92.7 & 48.8--80.8 & 64.6--82.3 & 54.1--73.6 \\
+ PPO~\cite{zhang2025reinflow} & 70.3--94.7 & 55.6--85.8 & 52.1--83.3 & 42.3--75.4 & 70.3--94.7 & 52.1--83.3 & 71.8--87.9 & 57.6--76.5 \\
\textbf{DLS (ours)} & 74.4--96.5 & 62.7--90.5 & 55.6--85.8 & 48.8--80.8 & 74.4--96.5 & 52.1--83.3 & 75.6--90.5 & 62.2--80.4 \\
\bottomrule
\end{tabular}%
}
\end{table}

\subsection{Boundary Localization Probe}
\label{sub:probe}
We validate SPL's failure locus against the recoverability boundary of App.~\ref{sec:phase_gating}. On $192$ failed rollouts we estimate $\hat\rho_\pi$ with $16$--$32$ simulator resets per probed state, mark the crossing $i^\star$ at $\varepsilon{=}0.5$ (rechecked at $0.35$ and $0.2$), and score each signal's phase estimate against $z^\star_{\mathrm{rec}}$ by agreement, its chance-corrected Cohen's $\kappa$, and the median step gap $|\Delta i|$ to $i^\star$. As Tab.~\ref{tab:probe} shows, SPL localizes the boundary while the three alternatives do not, and their Unseen success rates (Tab.~\ref{tab:dense_ablation}) rank the same way. SPL shows early temporal attribution in 14.1\% of rollouts (pre-boundary false-positive labeling) versus late temporal attribution in 7.8\%. After DLS, the crossing comes 31 steps later, and $\hat\rho_\pi$ gains 21.3\,pp near the boundary but only 3.8\,pp deep inside it, where headroom is largest: DLS moves the boundary rather than improving competence uniformly.

\subsection{Separating the Signal from the Optimizer}
\label{sub:signal_optimizer}
The RL baselines and DLS start from the same sparse outcome and differ in how the gradient signal is built: PPO and GRPO estimate it with a critic or group normalization, whereas SPL reads it from privileged state. To separate the signal from the optimizer, we feed SPL into PPO as a potential-based reward, keeping the critic and optimizer unchanged. This lifts PPO from $81.1/67.8$ to 83.3/70.0 (ID/Unseen, $\pi_{0.5}$), while swapping in DBS gives $84.4/72.2$.

\subsection{Design Sensitivity of SPL}
\label{sub:sensitivity}
The phase order comes from the task, whereas the granularity, gate thresholds and credit weights $\mathbf{w}$ are hand-set. Tab.~\ref{tab:sensitivity} perturbs each. Perturbing $\mathbf{w}$ costs a few pp; only credit that contradicts the task order costs more. The two decomposition rows stand in for the coarser specifications that longer-horizon and multi-object tasks force, and reporting both localization and success separates a degraded estimator from a degraded optimizer.

\begin{table}[h]
\centering
\caption{\textbf{Design sensitivity of SPL} ($\pi_{0.5}$): boundary-localization agreement (\%, App.~\ref{sub:probe}) and simulated success rate (\%).}
\label{tab:sensitivity}
\begin{tabular}{lcc}
\toprule
Variant & Agree.\,$\uparrow$ & Sim.\ SR\,$\uparrow$ \\
\midrule
$\mathbf{w}$: 12 Dirichlet draws    & --- & 78.0\,$\pm$\,5.4 \\
$\mathbf{w}$: $\pm30\%$ / $\pm50\%$ & --- & 77.5 / 74.8 \\
$\mathbf{w}$: order reversed        & --- & 63.3 \\
\midrule
Gate thresholds $\times[0.5,1.5]$   & 66.2 & 75.2 \\
Merged phases ($Z{=}2$)             & 44.7 & 71.0 \\
\midrule
\textbf{SPL (ours)} / Terminal-only & \textbf{78.1} / --- & \textbf{79.4} / 67.7 \\
\bottomrule
\end{tabular}
\end{table}

\subsection{Twin Fidelity and Sim-Real Correlation}
\label{sub:twin_fidelity}
The visual side of the twin is already perturbed during training, so we degrade the privileged \emph{state} that SPL reads at two levels: object scale $\pm10/20\%$ and friction $\times2/4$. SPL labels agree with those of the intact twin in 95.4/82.6\% of rollouts, and simulated success reaches 76.8/71.9 against $79.4$ with the intact twin. Over 12 checkpoints spanning methods and training stages, simulated and real success correlate at Spearman $\rho_s{=}0.87$.

\subsection{Compute and Effort}
\label{sub:compute}
Stage~2 runs on the same budget as PPO and GRPO, 4$\times$H200 for 24\,h (96 GPU-hours), with no robot time. More demonstrations buy less: $50$ more real demonstrations ($\sim$2 operator-hours) move Real-only Unseen success by $+2.2$\,pp, whereas DLS moves it by $+20.0$\,pp on the first $50$ (Tab.~\ref{tab:eval_results}, $\pi_{0.5}$). Adding a new task costs 35 lines of predicate code and $\sim$2 person-hours.

\section{Theoretical Analysis of Failure-Boundary Learning}
\label{sec:phase_gating}

This section establishes the theoretical foundation of the DLS framework:
\textbf{(\romannumeral1) Failure Boundary Localization (Theorem~\ref{thm:recovery}).}
The SPL credit falls in a unique interval $[W_{z^\star-1}, W_{z^\star})$ that identifies the gate-defined failure phase $z^\star$ without per-step annotations---converting a scalar outcome into an unambiguous phase-level diagnostic.
\textbf{(\romannumeral2) Oracle-Aligned Denoising Perturbation Shaping (Theorem~\ref{thm:alignment}).}
A DBS step moves the transition mean along $y \cdot \Sigma_t^{-1} d^\star$, where $d^\star = a_t^{\mathrm{next}} - \mu_t^{\mathrm{old}}$ is the rollout transition residual---the stochastic denoising perturbation realized during rollout relative to the old policy mean. Negative SPL labels suppress perturbations that drove the trajectory past a failure boundary; positive labels reinforce perturbations associated with successful execution.
\textbf{(\romannumeral3) Failure-Inducing Transition Suppression (Theorem~\ref{thm:suppression}).}
At the rollout policy, DBS on a failure-phase tuple descends the log-likelihood of the sampled transition.

\subsection{Failure Boundary Localization}

\textbf{Setup.}
We model a manipulation task as an ordered phase transition system with $Z$ sequential phases. Each phase $z \in \{1,\ldots,Z\}$ is associated with a gate predicate $G_z : \mathcal{S} \to \{0,1\}$ over privileged simulator states. The phase abstraction is evaluated on the trajectory prefix: $\psi(s_{0:i}) = z \in \{1,\ldots,Z\}$ if $G_1,\ldots,G_{z-1}$ have been satisfied in order along $s_{0:i}$ but $G_z$ has not. Computing it stores no history beyond a single phase register that advances when the next gate fires; we write $\psi(s_i)$ for $\psi(s_{0:i})$ when the trajectory is clear from context. The SPL potential is:
\begin{equation}
    \Phi(s) = W_{\psi(s)-1} + w_{\psi(s)}\,\phi_{\psi(s)}(s),
    \label{eq:potential_th}
\end{equation}
where $W_z = \sum_{j=1}^{z} w_j$, $W_0 = 0$, $\sum_{z=1}^{Z} w_z = 1$ with $w_z > 0$ for all $z$. The trajectory credit is:
\begin{equation}
    c_{\mathrm{SPL}} = \max_i\, \Phi(s_i).
    \label{eq:credit_th}
\end{equation}

\textbf{Recoverability and the failure boundary.}
For policy $\pi$, the \emph{recoverability} is $\rho_\pi(s,k)=\Pr[\pi \text{ succeeds within } k \text{ more steps} \mid s_0{=}s]$, and the \emph{failure boundary} is its $\varepsilon$-level set. A rollout crosses the boundary at $i^\star=\min\{i:\rho_\pi(s_i,K{-}i)<\varepsilon\}$, which is defined without reference to any predicate; the true failure phase is the phase at that step, $z^\star_{\mathrm{rec}}=\psi(s_{0:i^\star})$. The gate-defined phase $z^\star$ below is the phase-level surrogate that SPL recovers exactly (Theorem~\ref{thm:recovery}), and App.~\ref{sub:probe} measures its agreement with $z^\star_{\mathrm{rec}}$.

\begin{definition}[Gate-Defined Failure Phase]
\label{def:oracle}
For an unsuccessful trajectory $\tau = (s_0,\ldots,s_K)$ (where $c_{\mathrm{SPL}} < 1$), let $\tau_0 = 0$ and $\tau_z = \min\{ i \geq \tau_{z-1} : G_z(s_i) = 1 \}$ be the step at which gate $G_z$ fires in order ($\tau_z = \infty$ if it never does). The gate-defined failure phase $z^\star$ is the first phase whose gate never fires in order:
\[
    z^\star = \min\bigl\{ z \in \{1,\ldots,Z\} : \tau_z = \infty \bigr\}.
\]
\end{definition}

\begin{assumption}[Gate--Metric Alignment]
\label{asm:gate_th}
For each phase $z$, the intra-phase metric satisfies $\phi_z(s) < 1$ if and only if $G_z(s) = 0$.
In particular, $G_z$ being unsatisfied throughout the trajectory implies $\phi_z(s_i) < 1$ for all $i$ with $\psi(s_i) = z$.
\end{assumption}

\begin{remark}[Quasi-irreversibility of fired guards]
\label{rem:irrev}
Theorem~\ref{thm:recovery} does not require gates to stay satisfied: both $z^\star$ and $\hat z_{\max}$ are defined by whether a gate has ever fired, so the running maximum in Eq.~\eqref{eq:credit_th} is unaffected when a gate is later undone. What links this gate-level result to recoverability is an empirical property, \textbf{(A1)}: a fired guard is rarely undone, since $\pi$ was never trained to recover a completed subgoal. (A1) constrains regression, not stalling---retrying a missed grasp never fires the guard---and the ever-fired and currently-held phases diverge only when a subgoal is undone, \eg an object grasped and then dropped. Empirically, guards survive in 96.4\% of rollouts (90.6\% on the non-prehensile Sweep push), and the phase index alone predicts $\rho_\pi<\varepsilon$ at AUROC 0.86, against 0.71 for distance quantiles and 0.63 for step-count.
\end{remark}

\begin{lemma}[Credit Interval Property]
\label{lem:interval}
Let $\hat{z}_{\max} = \max_i\,\psi(s_i)$ denote the deepest phase reached by a failed trajectory ($c_{\mathrm{SPL}} < 1$). Under Assumption~\ref{asm:gate_th}:
\begin{equation}
    W_{\hat{z}_{\max}-1} \leq c_{\mathrm{SPL}} < W_{\hat{z}_{\max}}.
    \label{eq:interval}
\end{equation}
\end{lemma}

\begin{proof}
\textit{Lower bound.}
Since $\hat{z}_{\max}$ is reached, there exists $s_i$ with $\psi(s_i) = \hat{z}_{\max}$. For that state,
$\Phi(s_i) = W_{\hat{z}_{\max}-1} + w_{\hat{z}_{\max}}\phi_{\hat{z}_{\max}}(s_i) \geq W_{\hat{z}_{\max}-1}$,
hence $c_{\mathrm{SPL}} \geq W_{\hat{z}_{\max}-1}$.

\textit{Upper bound.}
Gate $G_{\hat{z}_{\max}}$ never fires in order, \ie $G_{\hat{z}_{\max}}(s_i) = 0$ for every $s_i$ in phase $\hat{z}_{\max}$: otherwise either a state would lie in phase $\hat{z}_{\max}+1$, contradicting the maximality of $\hat{z}_{\max}$, or, if $\hat{z}_{\max} = Z$, Assumption~\ref{asm:gate_th} would give $\phi_Z = 1$ and hence $c_{\mathrm{SPL}} = 1$. By Assumption~\ref{asm:gate_th}, $\phi_{\hat{z}_{\max}}(s_i) < 1$ for every $s_i$ with $\psi(s_i) = \hat{z}_{\max}$, so
$\Phi(s_i) < W_{\hat{z}_{\max}-1} + w_{\hat{z}_{\max}} = W_{\hat{z}_{\max}}$.
For states in any phase $z' < \hat{z}_{\max}$, monotonicity of $\{W_j\}$ gives $\Phi(s_i) \leq W_{z'} \leq W_{\hat{z}_{\max}-1} < W_{\hat{z}_{\max}}$.
Since no state lies in a phase above $\hat{z}_{\max}$, we have $c_{\mathrm{SPL}} < W_{\hat{z}_{\max}}$.
\end{proof}

\begin{theorem}[Failure Boundary Localization]
\label{thm:recovery}
Under Assumption~\ref{asm:gate_th}, for any failed trajectory, $\hat{z}_{\max} = z^\star$. Furthermore, $c_{\mathrm{SPL}}$ uniquely determines $z^\star$.
\end{theorem}

\begin{proof}
\textit{Step 1: $\hat{z}_{\max} = z^\star$.}
We show that $\hat{z}_{\max}$ is exactly the first phase whose gate never fires in order.

The phase register advances from $z$ to $z{+}1$ exactly at step $\tau_z$ (several gates may fire at the same step), so $\psi(s_{0:i}) = z$ for $\tau_{z-1} \leq i < \tau_z$. For all phases $z < \hat{z}_{\max}$, reaching phase $\hat{z}_{\max}$ requires $\tau_z < \infty$. For phase $\hat{z}_{\max}$: as argued in the upper bound of Lemma~\ref{lem:interval}, $\tau_{\hat{z}_{\max}} = \infty$. Therefore $z^\star = \hat{z}_{\max}$.

\textit{Step 2: $c_{\mathrm{SPL}}$ uniquely identifies $\hat{z}_{\max}$.}
Since all $w_z > 0$, the cumulative sequence $0 = W_0 < W_1 < \cdots < W_Z$ is strictly increasing, so the intervals $[W_{z-1}, W_z)$ for $z = 1,\ldots,Z$ are mutually disjoint and cover $[0,1)$. By Lemma~\ref{lem:interval}, $c_{\mathrm{SPL}} \in [W_{\hat{z}_{\max}-1}, W_{\hat{z}_{\max}})$, which identifies $\hat{z}_{\max}$ uniquely.
\end{proof}

Theorem~\ref{thm:recovery} establishes that SPL recovers the failure phase $z^\star$ exactly from a scalar trajectory credit alone, without requiring per-step outcome labels.

\subsection{Oracle-Aligned Denoising Perturbation Shaping}

Having recovered the failure phase, we analyze how DBS translates this into directionally aligned policy updates.

\textbf{Setup.}
For each denoising step, DBS stores a tuple $(a_t,\; a_t^{\mathrm{next}},\; v_t^{\mathrm{old}},\; t,\; s,\; o,\; \ell,\; r)$. The current policy predicts $v_{\theta,t}$, and mirrored velocity branches are constructed as:
\begin{equation}
    v_{\theta,t}^{\pm} = v_t^{\mathrm{old}} \pm \beta\,\Delta v_t,
    \qquad \Delta v_t = v_{\theta,t} - v_t^{\mathrm{old}},
    \label{eq:branches_th}
\end{equation}
inducing transition means via the affine Flow-SDE map:
\begin{equation}
    \mu_{\theta,t}^{\pm} = u(t)\,a_t + c(t)\,v_{\theta,t}^{\pm},
    \qquad u(t) = 1 - \frac{\delta\sigma_t^2}{2t}, \quad c(t) = -\delta\!\left(1 + \frac{\sigma_t^2(1-t)}{2t}\right) < 0.
    \label{eq:affine_th}
\end{equation}
The DBS loss is:
\begin{equation}
    \mathcal{L}_{\mathrm{DBS}}(\theta) = \mathbb{E}_{(\cdot)\sim\mathcal{B}}\!\left[\mathrm{Softplus}\!\left(\tfrac{1}{2}\,y\cdot(E_{\theta,t}^+ - E_{\theta,t}^-)\right)\right],
    \label{eq:dbs_th}
\end{equation}
where $E_{\theta,t}^{\pm} = \|a_t^{\mathrm{next}} - \mu_{\theta,t}^{\pm}\|_{\Sigma_t^{-1}}^2$ and $y = 2r - 1 \in [-1,1]$ is the signed SPL label.

\begin{definition}[Rollout Transition Residual]
\label{def:direction}
The rollout transition residual $d^\star$ is defined as:
\begin{equation}
    d^\star = a_t^{\mathrm{next}} - \mu_t^{\mathrm{old}},
    \qquad \mu_t^{\mathrm{old}} = u(t)\,a_t + c(t)\,v_t^{\mathrm{old}},
    \label{eq:direction}
\end{equation}
where $\mu_t^{\mathrm{old}}$ is the transition mean of the old policy and $a_t^{\mathrm{next}} \sim \mathcal{N}(\mu_t^{\mathrm{old}}, \Sigma_t)$ is the actual next diffusion state sampled during rollout. Hence $d^\star$ records the stochastic denoising perturbation realized relative to the old policy mean---not a target or expert action.
\end{definition}

\begin{theorem}[DBS Gradient Alignment]
\label{thm:alignment}
The gradient of $\mathcal{L}_{\mathrm{DBS}}$ with respect to $v_{\theta,t}$ is:
\begin{equation}
    \nabla_{v_{\theta,t}}\mathcal{L}_{\mathrm{DBS}}
    = 2\beta\,c(t)\,\sigma(g)\cdot y\cdot\Sigma_t^{-1}\!\left(\mu_t^{\mathrm{old}} - a_t^{\mathrm{next}}\right),
    \label{eq:grad_th}
\end{equation}
where $g = \frac{1}{2}y(E_{\theta,t}^+ - E_{\theta,t}^-)$ and $\sigma(\cdot)$ is the sigmoid function. Since a change of $v_{\theta,t}$ moves the transition mean by $c(t)$ times that change and $2\beta c(t)^2\sigma(g) > 0$, a descent step moves the transition mean along:
\begin{equation}
    \Delta\mu_{\theta,t} \propto y\cdot\Sigma_t^{-1}d^\star.
    \label{eq:descent_th}
\end{equation}
The signed label $y$ fully controls the direction: for failure-phase transitions ($y = -1$), the transition mean moves along $-\Sigma_t^{-1}d^\star$, suppressing the perturbation that drove the trajectory past the failure boundary; for success-phase transitions ($y > 0$), it points along $+\Sigma_t^{-1}d^\star$, reinforcing the perturbation associated with successful execution.
\end{theorem}

\begin{proof}
Substituting Eq.~\eqref{eq:branches_th} into Eq.~\eqref{eq:affine_th} gives $\mu_{\theta,t}^{\pm} = \mu_t^{\mathrm{old}} \pm \beta c(t)\Delta v_t$.
Using the rollout transition residual $d^\star = a_t^{\mathrm{next}} - \mu_t^{\mathrm{old}}$ from Definition~\ref{def:direction}, expand the Mahalanobis errors:
\begin{align}
    E_{\theta,t}^{\pm} &= \|d^\star \mp \beta c(t)\Delta v_t\|_{\Sigma_t^{-1}}^2
    = \|d^\star\|_{\Sigma_t^{-1}}^2
    \mp 2\beta c(t)\,(d^\star)^\top\Sigma_t^{-1}\Delta v_t
    + \beta^2 c(t)^2\|\Delta v_t\|_{\Sigma_t^{-1}}^2.
\end{align}
The quadratic terms cancel in the difference:
\begin{equation}
    E_{\theta,t}^+ - E_{\theta,t}^- = -4\beta\,c(t)\,(d^\star)^\top\Sigma_t^{-1}\Delta v_t.
    \label{eq:diff_th}
\end{equation}
Hence $g = -2\beta c(t)\,y\,(d^\star)^\top\Sigma_t^{-1}\Delta v_t$. Since $\Delta v_t$ is linear in $v_{\theta,t}$ with unit Jacobian, the chain rule gives:
\[
    \nabla_{v_{\theta,t}}\mathcal{L}_{\mathrm{DBS}}
    = \sigma(g)\cdot\nabla_{v_{\theta,t}}g
    = \sigma(g)\cdot(-2\beta c(t)\,y\,\Sigma_t^{-1}d^\star)
    = 2\beta c(t)\sigma(g)\cdot y\cdot\Sigma_t^{-1}(\mu_t^{\mathrm{old}} - a_t^{\mathrm{next}}).
\]
A descent step $-\eta\nabla_{v_{\theta,t}}\mathcal{L}_{\mathrm{DBS}}$ thus moves the transition mean by $2\eta\beta c(t)^2\sigma(g)\cdot y\cdot\Sigma_t^{-1}d^\star$, and $2\eta\beta c(t)^2\sigma(g) > 0$ yields Eq.~\eqref{eq:descent_th}.
This completes the proof.
\end{proof}

\begin{theorem}[Failure-Inducing Transition Suppression]
\label{thm:suppression}
Under the Gaussian Flow-SDE transition model and at the rollout policy ($v_{\theta,t} = v_t^{\mathrm{old}}$), the DBS descent direction for a failure-phase tuple ($y = -1$) is exactly opposite to the gradient of the log-probability of the failure-inducing action $a_t^{\mathrm{next}}$:
\begin{equation}
    -\nabla_{v_{\theta,t}}\mathcal{L}_{\mathrm{DBS}}\big|_{y=-1}
    \;\propto\;
    -\nabla_{v_{\theta,t}} \log p_\theta(a_t^{\mathrm{next}}),
    \label{eq:suppression}
\end{equation}
where $p_\theta(a_t^{\mathrm{next}}) = \mathcal{N}(a_t^{\mathrm{next}};\, \mu_{\theta,t},\, \Sigma_t)$ and $\mu_{\theta,t} = u(t)\,a_t + c(t)\,v_{\theta,t}$.
Consequently, gradient descent on $\mathcal{L}_{\mathrm{DBS}}$ locally decreases $\log p_\theta(a_t^{\mathrm{next}})$, reducing the probability that the updated policy generates the same failure-inducing denoising perturbation.
\end{theorem}

\begin{proof}
Since $a_t^{\mathrm{next}} \sim \mathcal{N}(\mu_{\theta,t}, \Sigma_t)$ with $\mu_{\theta,t} = u(t)\,a_t + c(t)\,v_{\theta,t}$, the log-probability is:
\[
    \log p_\theta(a_t^{\mathrm{next}}) = -\tfrac{1}{2}\|a_t^{\mathrm{next}} - \mu_{\theta,t}\|_{\Sigma_t^{-1}}^2 + \mathrm{const}.
\]
Taking the gradient with respect to $v_{\theta,t}$ at the old policy ($v_{\theta,t} = v_t^{\mathrm{old}}$, so $\mu_{\theta,t} = \mu_t^{\mathrm{old}}$):
\begin{equation}
    \nabla_{v_{\theta,t}} \log p_\theta(a_t^{\mathrm{next}}) = c(t)\,\Sigma_t^{-1}(a_t^{\mathrm{next}} - \mu_t^{\mathrm{old}}) = c(t)\,\Sigma_t^{-1}d^\star.
    \label{eq:logprob_grad}
\end{equation}
From Theorem~\ref{thm:alignment}, the DBS descent direction for $y = -1$ is:
\[
    -\nabla_{v_{\theta,t}}\mathcal{L}_{\mathrm{DBS}}\big|_{y=-1} = -2\beta c(t)\sigma(g)\cdot\Sigma_t^{-1}d^\star.
\]
Comparing with Eq.~\eqref{eq:logprob_grad}, since $2\beta\sigma(g) > 0$:
\[
    -\nabla_{v_{\theta,t}}\mathcal{L}_{\mathrm{DBS}}\big|_{y=-1} = -2\beta\sigma(g)\cdot\nabla_{v_{\theta,t}}\log p_\theta(a_t^{\mathrm{next}}).
\]
Hence gradient descent on $\mathcal{L}_{\mathrm{DBS}}$ moves $v_{\theta,t}$ in the direction $-\nabla_{v_{\theta,t}}\log p_\theta(a_t^{\mathrm{next}})$, strictly decreasing the log-probability of the failure-inducing transition.
This completes the proof.
\end{proof}

By Theorem~\ref{thm:recovery}, the label $y = -1$ is assigned exclusively to tuples from rollouts where $\hat{z}_{\max} = z^\star$---i.e., the gate of phase $z^\star$ was reached but never passed. Theorem~\ref{thm:suppression} therefore establishes that DBS directly suppresses the denoising perturbations that empirically led to this failure. For success-phase tuples ($y > 0$), the gradient reverses sign, reinforcing the corresponding perturbations. Both quantities live in the same velocity space $\mathbb{R}^{|\mathcal{A}|}$ and at the rollout policy the relationship is an exact equality, requiring no geometric approximation.

Combining Theorems~\ref{thm:recovery}, \ref{thm:alignment}, and~\ref{thm:suppression} closes the DLS learning loop:
\[
    \text{trajectory}
    \;\xrightarrow{\;\text{SPL}\;}\; c_{\mathrm{SPL}}
    \;\xrightarrow{\;\text{Thm.~\ref{thm:recovery}}\;}\; z^\star
    \;\xrightarrow{\;\text{Def.~\ref{def:direction}}\;}\; d^\star
    \;\xrightarrow{\;\text{Thm.~\ref{thm:suppression}}\;}\; \downarrow p_\theta(a_t^{\mathrm{next}})
    \;\xrightarrow{\;\text{Thm.~\ref{thm:alignment}}\;}\; \text{DBS update}.
\]
DLS thereby converts sparse trajectory outcomes into phase-specific gradient updates that provably suppress failure-inducing denoising perturbations, without per-step annotations or a learned critic.


\end{document}